\documentclass[journal]{IEEEtran}
\usepackage{amsmath,amsfonts, amsthm}
\usepackage{algorithmic}
\usepackage{algorithm}
\usepackage{array}
\usepackage[caption=false,font=normalsize,labelfont=sf,textfont=sf]{subfig}
\usepackage{textcomp}
\usepackage{stfloats}
\usepackage{url}
\usepackage{verbatim}
\usepackage{graphicx}
\usepackage{cite}
\newtheorem{theorem}{Theorem}[section]
\newtheorem{lemma}[theorem]{Lemma}

\DeclareMathOperator*{\argmax}{arg\,max}

\begin{document}

\title{Fully Distributed Fog Load Balancing with Multi-Agent Reinforcement Learning}

\author{Maad Ebrahim and Abdelhakim Hafid,~\IEEEmembership{Member,~IEEE,}%
\thanks{M. Ebrahim and A. S. Hafid are with the NRL, Department of Computer Science and Operational Research, University of Montreal, Montreal, QC H3T-1J4, Canada (e-mail: maad.ebrahim@umontreal.ca; ahafid@iro.umontreal.ca).}%
\thanks{Corresponding author: Maad Ebrahim (maad.ebrahim@umontreal.ca).}%
}

\markboth{IEEE Transactions on Network and Service Management}%
{Ebrahim \MakeLowercase{\textit{et al.}}: Fully Distributed Fog Load Balancing with Multi-Agent Reinforcement Learning}


\maketitle

\begin{abstract}
Real-time Internet of Things (IoT) applications require real-time support to handle the ever-growing demand for computing resources to process IoT workloads. Fog Computing provides high availability of such resources closer to IoT devices in a distributed manner. However, these resources must be efficiently managed to distribute unpredictable IoT traffic demands among heterogeneous Fog resources. To fully leverage these distributed Fog resources, their management must be done in a fully distributed manner, i.e., without any source of centralization management. Hence, this paper proposes a fully distributed load-balancing solution with Multi-Agent Reinforcement Learning (MARL) that intelligently distributes IoT workloads to minimize waiting delay while providing fair resource utilization in the Fog network, enhancing end-to-end execution delay. By leveraging distributed decision-making without centralized or inter-agent coordination, MARL agents can provide faster and better load distribution decisions with minimal overhead compared to a single centralized agent solution. Besides performance gain, a fully distributed solution allows for scalable Fog load balancing, where agents manage subsets of a global Fog network in smaller collaboration regions. To ensure global optimality, the agents learn to optimize a common objective, i.e., aiming to reduce workload accumulation in every Fog node. Reducing state and action spaces for each agent to a few Fog nodes in proximity makes learning easier and convergence faster. Additionally, learning independent policies by avoiding coordination encourages the agents to leverage Fog resources in proximity before seeking remote resources. In this paper, we also evaluate the impact of a realistic frequency for observing the environment instead of the common assumption of having observations readily available for every action, highlighting the trade-off between realism and performance and the solution's applicability to real-world deployment constraints. 
\end{abstract}


\section{Introduction}
\label{sec:intro}

In the era of interconnected devices, Internet of Things (IoT) has emerged as a pervasive technology, facilitating seamless communication and data exchange among various devices. However, the surge in IoT devices has posed significant challenges in efficiently managing and processing the generated data. Fog computing has emerged as a promising paradigm to address these challenges by distributing computing resources closer to data sources. Unlike the Cloud, Fog nodes can serve time-sensitive and computationally intensive applications of resource-limited IoT devices. 

Fog Computing necessitates efficient resource management due to its distributed nature and latency-sensitive applications. Distributing IoT load efficiently between available Fog resources is essential to reduce end-to-end latency and enhance the system performance. High workload generation rates can increase the waiting delay due to the accumulation of tasks in the queues of Fog nodes. Hence, minimizing the waiting delay is the key to reducing the overall end-to-end execution delay of IoT workloads. 

Load balancing in Fog computing environments plays a crucial role in ensuring optimal resource utilization and timely processing of IoT data; this allows for minimizing task completion time and improving overall network performance. Traditional load-balancing techniques often struggle with the dynamic nature of Fog environments, characterized by limited heterogeneous resources and fluctuating workload demands. Reinforcement Learning (RL) can be used to learn optimal load distribution policies by interacting with such dynamic environments. To take full advantage of the distributed computing capabilities of Fog systems, load distribution decisions must be provided in a distributed manner closer to IoT devices at the network edge \cite{ebrahim2022blockchain}. 

Fully distributed Multi-Agent RL (MARL) overcomes the scalability limitations of centralized solutions by allowing autonomous agents to independently manage smaller sets of Fog nodes in proximity instead of managing a larger network by a single agent. These gents incorporate self-adaptation by leveraging the practical lifelong transfer learning framework proposed in our previous work \cite{ebrahim2023lifelong}, allowing them to adapt to dynamic environmental changes seamlessly. Each agent is deployed in an IoT Access Point (AP), i.e., Base Station, to learn its own load distribution strategy for its associated IoT devices considering only the candidate Fog nodes in its region. Managing the same set of candidate Fog nodes by multiple agents forms a collaboration region (see Fig. \ref{fig:regions}), where the goal is to optimize these shared resources independently by each agent. These collaboration regions can be isolated, but they can also be overlapped by sharing candidate Fog nodes between collaboration regions as in Fig. \ref{fig:regions}. 

\begin{figure}[!htbp]
\centering
\includegraphics[trim=130 10 120 15, clip, width=0.48\textwidth]{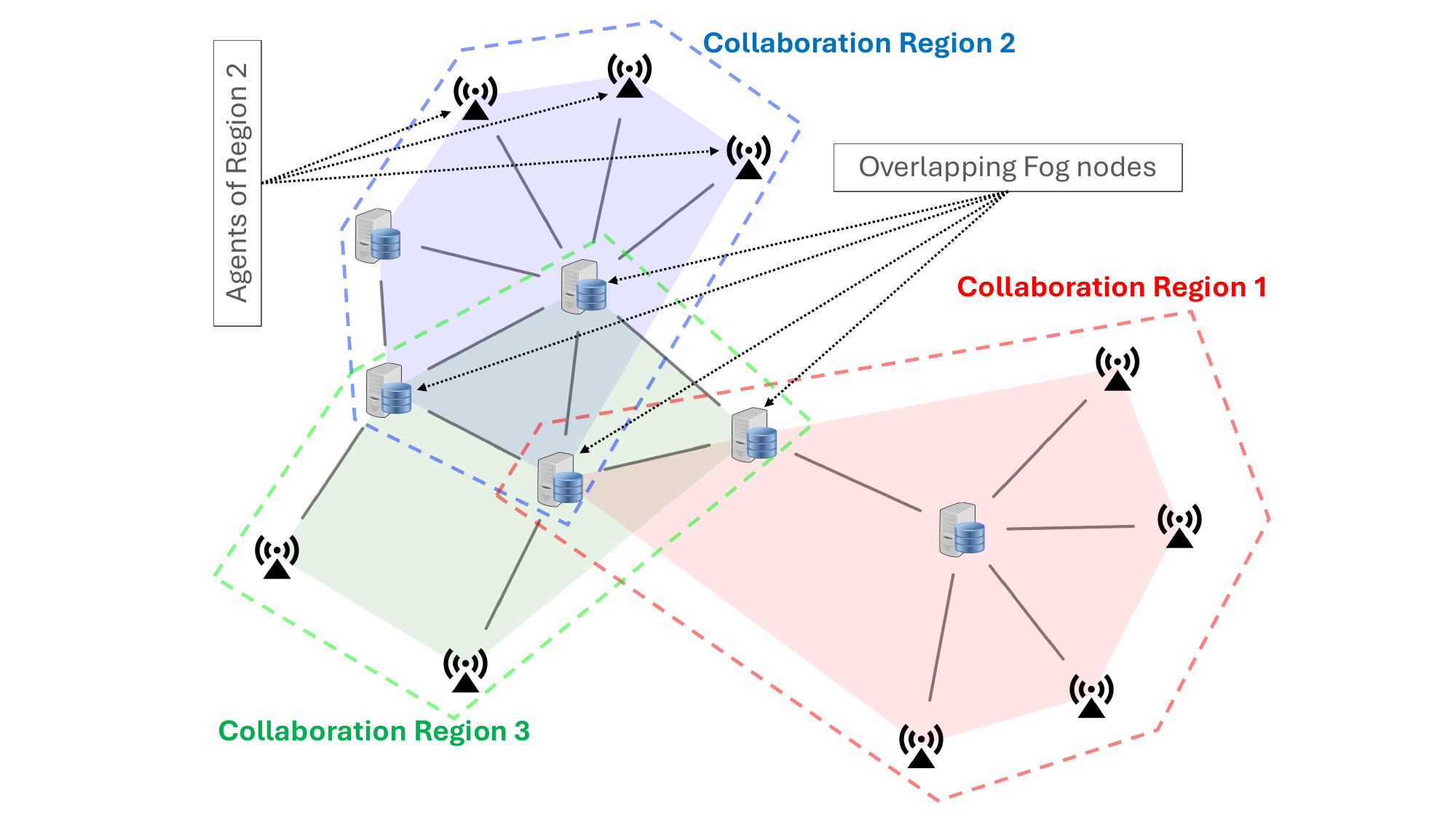}
\caption{A simplified Fog network with overlapping collaboration regions.}
\label{fig:regions}
\end{figure}

Considering small collaboration regions lowers the state and action spaces of individual agents, reducing the computational complexity of the problem and making the learning task easier for each agent. To simplify the problem further, we allow the agents to learn independent policies without collaboration, where they learn to take independent actions that fit their unique load requirements and the Fog resources closer to each agent. Learning separate load distribution strategies can be faster than learning a shared strategy for all agents, but requires a way to ensure global optimality while avoiding making the environment non-stationary for each agent. To do this in both isolated and overlapped regions, all agents learn to optimize a common objective through a common reward function, designed to reduce system-wide workload accumulation by learning to select less loaded Fog nodes.

To avoid learning a common load distribution strategy, agents must learn in a fully distributed manner without sharing any information with other agents in the system, either through a centralized coordination entity or inter-agent communication. This is because coordination eventually leads to learning a common policy, either by explicitly learning a shared policy with distributed execution of that policy or by implicitly learning a common policy through parameter and experience sharing. Parameter sharing is done by exchanging the learned parameters of the neural network (knowledge) between agents. Experience sharing can be done by exchanging the actions or the full interactions in the replay buffer. 

Besides leading to a common policy, coordination causes significant communication overhead, especially with inter-agent communication with high information-sharing frequency. Most MARL solutions in the literature consider centralized collaboration through one or more centralized coordinators, reproducing the scalability problem we are trying to solve in the first place. Moreover, the overhead of observing Fog information from across the network using multiple agents is never considered in the literature. Researchers commonly assume this information is readily available in real-time for every agent's action, which is impractical in network environments. Although this unrealistic assumption leads to theoretically optimal results, we evaluate in this work a more realistic way to observe Fog information using interval-based observations. We consider observation intervals of existing multi-casting protocols to evaluate a more realistic, although sub-optimal, performance of our solution. 

This paper presents a novel fully distributed, efficient, and self-adapting MARL load distribution solution for realistic global-scale Fog Computing implementations without centralized or inter-agent coordination. Through empirical evaluation, we show the effectiveness of our distributed MARL-based solution compared to the centralized RL solution proposed in our previous work \cite{ebrahim2023lifelong}. We also show the performance of observing the environment using realistic multi-cast intervals of 3 seconds compared to using unrealistic real-time observations. Hence, the key contributions of this paper can be summarized as follows:
\begin{itemize}
    \item Propose a scalable Fog load balancing solution for a global-scale implementation using fully distributed RL agents that blindly collaborate without coordination by learning to optimize a common objective through independent actions.
    \item Problem complexity is reduced by grouping the agents into smaller collaboration regions, with isolated or overlapped Fog nodes between the regions, to observe and manage smaller subsets of Fog nodes from the global network.
    \item These distributed agents reduce training time by incorporating a lifelong self-adaptation framework using transfer learning, allowing the agents to efficiently distribute dynamically changing IoT load across the Fog environment.
    \item Evaluate the realism-performance trade-off based on the frequency of observing Fog information from across the network, where we compare the common practice in the literature of unrealistically assuming Fog information is readily available in real-time for every action against the more realistic approach of collecting Fog information using existing interval-based multi-casting protocol, like Gossip.
\end{itemize}

The rest of the paper is organized as follows. Section \ref{sec:related} presents the related work. Section \ref{sec:system} presents the system model. Section \ref{sec:DRL} presents the proposed approach. Section \ref{sec:eval} presents the evaluation of our proposed approach. Finally, Section \ref{sec:conclusion} concludes the paper.

\section{Related Work}
\label{sec:related}
In recent years, researchers have made significant contributions toward developing efficient resource management solutions for Fog Computing environments. One such approach is centralized RL, where a single agent manages the entire network. However, centralization causes scalability limitations and single-point-of-failure vulnerabilities. Using a centralized load-balancing node in large-scale Fog networks increases the time needed to make the load assignment decision, increasing end-to-end execution delay. Mitigating the limitations of single-agent RL solutions requires distributed multi-agent solutions, where multiple agents collaborate to manage Fog resources across the network.

Jain {\it et al.\/}~\cite{QoSAware} proposed training an agent for each Fog node to mitigate the delay of decision-making in centralized solutions. However, they trained all the agents offline on a centralized cloud data center to later distribute the trained models to each Fog node, making offloading decisions locally with minimal delay. Centralized training provides the same decision-making policy for all agents despite their varying requirements, conditions, and nearby resource availability. Also, centralized training, and retraining in case of changes in the environment, require collecting local experience trajectories of each agent, i.e., observations, actions, rewards, and next observations. Transferring such information to a centralized entity to retrain these agents is inefficient compared to local retraining, especially for highly dynamic environments that keep changing over time. 

Determining the optimal timing for retraining the agents in a centralized entity presents a critical trade-off between decision accuracy and the significant retraining overhead. In addition, an RL decision can only be made when the observation of the current step is available, along with the reward of the previous step. Observing the state, for each decision, of every Fog node in large-scale Fog networks will cause a significant delay, especially with frequently generated workloads \cite{scaleproblem}. This delay is exponentially proportional to the number of agents in a distributed solution, making it impractical in real-world network deployment either with a centralized or a distributed solution. This delay is often not considered in the literature for Fog load-balancing solutions, assuming observations are readily available for every action at no cost. Without considering the time needed to transfer Fog load provisioning information to all RL agents, the effectiveness of these solutions in practice can not be guaranteed. 

Gao {\it et al.\/}~\cite{largeScale} used MARL to address the scalability problem of RL-based load-balancing solutions. They provide task offloading and resource allocation in large-scale heterogeneous Multi-access Edge Computing (MEC) systems. In their design, each edge server is regarded as an actor agent while still considering a centralized critic, where the critic is a neural network used to evaluate the actions of actor networks \cite{sutton2018reinforcement}. This centralized critic facilitates the coordination of offloading policies among those distributed actors, leading to a common policy. Besides not being fully distributed, the authors assumed that observations are readily available when the agents need to make assignment decisions, where the delay of exchanging Fog status information to every agent is ignored.

Baek {\it et al.\/}~\cite{FLoadNet} proposed a centralized critic network with distributed individual actor networks. Those distributed agents are located in APs and they cooperatively learn to solve the joint link and server load-balancing problem. A common policy for the distributed cooperative agents is optimized using parameter-sharing and distributed execution; the agents communicate among themselves while evaluating a common value function. The authors assume that updated policy parameters are shared between the agents by exchanging communication messages over limited bandwidth channels. However, they did not show the transmission delay associated with those communication messages, which augments when scaling up the solution for a larger-scale Fog environment. Such a delay can have a significant impact on the collaboration of the agents due to sharing outdated parameters.

Similarly, Goudarzi {\it et al.\/}~\cite{muDDRL} proposed a MARL framework to improve the execution delay of heterogeneous Fog services using multiple actors and a centralized critic. They deployed distributed agents (they call them brokers) that work in parallel to make service offloading decisions across different Fog nodes. A centralized service offloading model, i.e., the learner of the target policy, is trained using experience trajectories generated by the distributed actors. But again, they did not consider the delay caused by exchanging such information in a large-scale environment, especially when request rates are high. In addition, they did not consider the delay when updating the value function of each actor from the common centralized critic network.

Wei {\it et al.\/}~\cite{ManytoMany} proposed a scheme to improve vehicular Fog resource utilization in large-scale Vehicular Fog Computing (VFC) networks. They used MARL to decide on multi-tier task offloading in a distributed manner. Like most MARL-based solutions in the literature, they used the approach of centralized training with decentralized execution (CTDE). They share the agents' parameters with a joint critic network using the Attention machine learning mechanism, leading to a joint network that incorporates the learning process of distributed gated actors that exploit local computational potentials. Likewise, Suzuki {\it et al.\/}~\cite{CooperativeOffloading2} used CTDE to send learning samples from each agent to a centralized training entity, including task information, network link utilization, and Fog resources utilization. They unrealistically assumed all this information readily available in real-time for all agents when making offloading decisions.

Zhang {\it et al.\/}~\cite{onlineRequestScheduling} also used CTDE-based MARL to reduce the long-term average delay and energy consumption while improving the throughput of dynamic Edge networks using online Request Scheduling. Centralized training was adopted to improve the long-term system performance by realizing the implicit cooperation among edge nodes. This cooperation was obtained by transmitting the replay buffer data of each agent into a shared training server to train a common policy. During distributed execution, on the other hand, only local environment information is needed for each edge node to make the request scheduling decision. However, the authors did not discuss the scope of local information and the delay associated with transmitting such information over the network. In a large-scale network, such a delay is not tolerable for real-time decisions required by real-time applications. 

Lu {\it et al.\/}~\cite{DynamicScheduling} proposed MARL-based dynamic scheduling to handle the massive number of service requests by the ever-increasing number of vehicles in Internet-of-Vehicle (IoV) networks. They leveraged CTDE to learn high-quality scheduling rules for the service function scheduling problem. In their architecture, the generated scheduling actions of the distributed policies are evaluated by separate critic networks deployed in a centralized controller node. Similarly, Cui {\it et al.\/}~\cite{CooperativeOffloading3} proposed offline CTDE to cooperatively optimize the offloading strategy for computation-intensive and delay-sensitive IoV applications. Offline CTDE helped reduce the communication between the agents while increasing the stability of the environment. However, offline execution does not allow the agents to dynamically adapt to environment changes. In addition, their solution will not scale to larger networks as all their agents must observe the global state of the environment, which is assumed, unrealistically, readily available for every decision to be made. 

CTDE was also used by Gao {\it et al.\/}~\cite{adaptiveOffloading} to predict the queuing delay of edge servers by training a centralized critic in the Cloud. To evaluate each agent's policy, the critic network needs to collect observations and their corresponding actions from all agents in the system. Then, the agents independently make their decisions using their policies distributively. Mangiaracina {\it et al.\/}~\cite{dataAAS} proposed adaptive MARL to manage data placement among Fog nodes considering Quality of Service (QoS) requirements. Unlike CTDE, they enabled cooperation between multiple agents using the concept of a shared repository accessible by all agents. This repository contains knowledge bases and event logs shared between those agents to improve their decision-making. The authors claim that this centralized repository does not affect the overall performance as a bottleneck in the architecture. However, they did not consider the time each agent needs to access information from this repository.

To coordinate the decisions of independent agents, Yang {\it et al.\/}~\cite{CooperativeOffloading1} proposed that the decisions of each agent must depend on shared global state information. They proposed a centralized migration scheduling model to collect server load information and calculate the migration likelihood. They estimate the global information in a centralized model to reduce the communication overhead caused by broadcasting this information from each agent to the centralized critic network. They showed a 20\% gap in performance between their approach and the optimal solution but with only half the decision-making time, where the optimal solution is a centralized decision-making method having the actions of each agent and the actual state information of each Fog node (not estimated) available for all agents unrealistically in real-time.

Based on the literature discussed above, most existing MARL solutions adopt a centralized training approach and rely on unrealistic assumptions regarding collecting observations (see Table \ref{tab:compare}). The centralized component in these solutions raises several challenges, especially in large-scale environments, including computational complexity, exploration of larger state and action spaces, and the communication overhead due to exchanging training information with the centralized entity. A truly distributed solution remains preferable to enhance scalability and decision-making time, where agents can learn and make decisions independently without centralized or inter-agent coordination. Allowing the agents to learn independent policies based on the unique requirements of every agent can provide faster convergence compared to learning a common policy for all the agents.

\begin{table*}[tb]
\footnotesize
\caption{Comparison of state-of-the-art MARL-based Fog load balancing solutions. \label{tab:compare}}
\centering
\begin{tabular}{|l|c|c|c|c|c|c|}
\hline
\textbf{SOTA} & \textbf{Training} & \textbf{Retraining} & \textbf{Information Sharing} & \textbf{ Observation Availability} \\
\hline
Jain {\it et al.\/}~\cite{QoSAware} & Centralized & No & A shared trained model & Unrealistically instant \\
\hline
Gao {\it et al.\/}~\cite{largeScale} & Centralized & No & Parameter Sharing (Centralized) & Unrealistically instant \\
\hline
Baek {\it et al.\/}~\cite{FLoadNet} & Centralized & No & Parameter Sharing (Inter-agent) & Unrealistically instant \\
\hline
Goudarzi {\it et al.\/}~\cite{muDDRL} & Centralized & No & Experience Sharing (Centralized) & Unrealistically instant \\
\hline
Wei {\it et al.\/}~\cite{ManytoMany} & Centralized & No & Parameter Sharing (Centralized) & Unrealistically instant \\
\hline
Zhang {\it et al.\/}~\cite{onlineRequestScheduling} & Centralized & No & Experience Sharing (Centralized) & Unrealistically instant \\
\hline
Lu {\it et al.\/}~\cite{DynamicScheduling} & Centralized & No & Experience Sharing (Centralized) & Unrealistically instant \\
\hline
Cui {\it et al.\/}~\cite{CooperativeOffloading3} & Centralized & No & Experience Sharing (Centralized) & Unrealistically instant \\
\hline
Gao {\it et al.\/}~\cite{adaptiveOffloading} & Centralized & No & Experience Sharing (Centralized) & Unrealistically instant \\
\hline
Suzuki {\it et al.\/}~\cite{CooperativeOffloading2} & Centralized & No & Experience Sharing (Centralized) & Unrealistically instant \\
\hline
Mangiaracina {\it et al.\/}~\cite{dataAAS} & \textbf{Distributed} & No & Shared knowledge base & Unrealistically instant \\
\hline
Yang {\it et al.\/}~\cite{CooperativeOffloading1} & Centralized & No & A shared trained model & Unrealistically instant \\
\hline
\hline
\textbf{Our fully distributed approach} & \textbf{Distributed} & \textbf{Lifelong learning} & \textbf{None} & \textbf{Interval-based observations} \\
\hline
\end{tabular}
\end{table*}

To fill this gap in the literature, this paper proposes a fully distributed solution using independent agents that efficiently distribute IoT load in Fog networks without reliance on centralized components or inter-agent communication. Besides mitigating communication overhead, avoiding coordination between the agents leads to learning separate decision-making policies, which is easier than learning a common load distribution strategy. These independent agents learn to reduce workload accumulation in Fog nodes by receiving a reward for every load assignment decision representing the accumulated load in the selected Fog node. Optimizing this common objective encourages each agent to leverage nearby resources before seeking remote resources, mitigating the effect of the non-stationarity problem of independently trained agents. 

This reward function allows for a global-scale implementation when forming smaller groups of agents each managing a small subset of nearby Fog nodes. This grouping reduces state and action spaces for the agents, making learning easier and state observation overhead smaller. To reduce the overhead of observing Fog information further, we present the first attempt to evaluate collecting state observations and rewards using an interval-based Gossip multi-casting protocol. This approach is more realistic than assuming such information is readily available immediately before every decision. This evaluation explores the trade-off between optimality and applicability of the solution in real-world networks, leading to a deployment-ready solution. 

\section{The Fog System Model}
\label{sec:system}
This section defines the system model of the Fog network and the formulation of the load balancing problem. Table \ref{tab:symbols} shows the definitions of the mathematical symbols used throughout this paper. We reproduce the realistic Fog environments used in our previous works \cite{EBRAHIM2023privacy, EBRAHIM2023resilience} to compare the proposed fully distributed (multi-agent) solution with the centralized (single agent) solution in \cite{EBRAHIM2023privacy}.

\begin{table}[tb]
\footnotesize
\caption{List of mathematical notations. \label{tab:symbols}}
\centering
\begin{tabular}{|l|l|}
\hline
\textbf{Symbol} & \textbf{Definition}\\
\hline
$N$ & Set of AP, Fog, and Cloud nodes ($N^{AP}, N^F, N^C$)\\
\hline
$N^{F_C}$ & Candidate Fog nodes for the agents in a given region \\
\hline
$IPT_n$ & Compute capabilities of node $n$ (Instructions/timestep) \\
\hline
$RAM_n$ & Memory capabilities of node $n$ \\
\hline
$L$ & Set of Network links connecting the nodes \\
\hline
$BW_l$ & Bandwidth of link $l$ \\
\hline
$PR_l$ & Propagation delay of link $l$ \\
\hline
$DA$ & Set of distributed applications \\
\hline
$M_{da}$ & Set of modules of application $da$ \\
\hline
$D_{da}$ & Set of dependencies between the modules of $da$ \\
\hline
$w$ & A category from the set of workload categories $W$ \\
\hline
$\beta$ & The scale parameter of the exponential distribution \\
\hline
$\mathbb{Q}_n$ & The number of queued jobs in node $n$ \\
\hline
$s_t$ & State at step $t$, from a set of possible states $S$ \\
\hline
$a_t$ & Action at step $t$, from a set of possible actions $A$ \\
\hline
$r_t$ & The agent's immediate reward at step $t$ \\
\hline
$G$ & The discounted sum of immediate rewards (return) \\
\hline
$\gamma$ & The discount factor of the return \\
\hline
$\mathcal{F}$ & Set of candidate Fog nodes in a collaboration region $\mathcal{R}$ \\
\hline
$\mathcal{F}_S$ & Shared Fog nodes between two collaboration regions \\
\hline
$\pi(s, a)$ & The agent policy for taking action $a$ in state $s$ \\
\hline
$Q_\pi(s, a)$ & The value for using $\pi$ after taking action $a$ in state $s$ \\
\hline
$Q^*$ & The optimal $Q$-Value for using the optimal policy $\pi^*$ \\
\hline
$\phi(s, a)$ & A mapping of the state-action pairs \\
\hline
$\theta$ & The neural network weights for the state-action pairs \\
\hline
$\alpha$ & The learning rate for updating the $Q$-Value \\
\hline
$\max_aQ$ & The maximum $Q$-Value in state $s$ for all actions \\
\hline
$\argmax_aQ$ & The action that maximizes the $Q$-Value in state $s$ \\
\hline
$E_\mathcal{T}$ & The new environment (target task $\mathcal{T}$) \\
\hline
$\mathcal{I}_\mathcal{S}$ & Information from the source task $\mathcal{S}$ (old environment) \\
\hline
$\mathcal{B}$ & Experiences stored in the agent's replay buffer \\
\hline
\end{tabular}
\end{table}

The Fog environment has a mesh-like topology, with a central cloud, interconnected Fog nodes, and APs that work as gateways for their associated IoT devices. Load-balancing agents are deployed in these APs, where each agent makes its own decisions to distribute the load of its associated IoT devices between the Fog nodes in its collaboration region, i.e., its candidate Fog nodes. A group of agents in each collaboration region manages a set of candidate Fog nodes, requiring the agents to jointly optimize the resources of all candidate Fog nodes in that region. To be fully distributed, the agents must collaborate blindly without exchanging any kind of coordination information, like experiences, actions, or policies parameters.

In our simulated environments, the Fog network is defined by the set of $N = n_1, n_2, \cdots, n_z$ nodes described by their compute $IPT_n$ and memory $RAM_n$ resources \cite{EBRAHIM2023privacy, EBRAHIM2023resilience}. The Nodes, including $N^{AP}$ AP nodes, $N^F$ Fog nodes, and $N^C$ Cloud nodes, are connected through $L = l_1, l_2, \cdots, l_z$ bidirectional links. Each link is characterized by the pair of nodes ($n_i$, $n_j$) that it connects, its bandwidth $BW_l$, and its propagation delay $PR_l$. A set of distributed applications $DA = da_1, da_2, \cdots, da_z$ simultaneously run in the system; each application has a set of modules $M_{da}$ and dependency messages $D_{da}$ between these modules, i.e., workloads or requests.

Nodes and network links have heterogeneous resources and must serve different workload categories $|W|$ each with different compute requirements, i.e., light, medium, and heavy workloads. Workloads are generated as a Poisson Point Process using exponential distributions with three different scale parameters $\beta$ to simulate high, medium, and low workload generation rates. This allows for a dynamic representation of job arrivals in our Fog environment, where small $\beta$ values increase the probability of having smaller time intervals between subsequent workloads, leading to higher generation rates. Nodes and links are modeled with M/M/1 queuing models to handle the generated workloads, giving the ability to provide insights into performance metrics such as queue lengths and waiting times. To minimize the waiting time before being served by Fog nodes, the agents must learn to reduce the number of queued jobs $\mathbb{Q}_n$ in the queue of each Fog node $n$. Reducing workload accumulation in each Fog node reduces Fog overload probability and the end-to-end execution delay of IoT application loops \cite{EBRAHIM2023privacy}.

To optimize this objective by independent agents, the agents need to learn independent policies that can balance the generated load between Fog nodes. To do that without exchanging coordination information, the agents observe the number of jobs queued in each Fog node in their region (as environment state $s$) to select the Fog node with the smallest load (as action $a$). Each agent takes its actions, from the set of possible actions $A$ corresponding to the set of candidate Fog nodes $N^F$ in its region, independently from all other agents. To reduce waiting time collaboratively by all agents, they learn to minimize the common objective of reducing workload accumulation in all Fog nodes. Therefore, the reward $r$ that evaluates their most recent action is defined by the load in the most recently assigned Fog node by that action, encouraging the agents to select the least loaded Fog nodes collaboratively. This rewarding mechanism allows the agents to jointly reduce workload accumulation in the Fog nodes of their region without sharing any collaboration information.

\section{Distributed RL Load Balancing}
\label{sec:DRL}

A fully independent Double Deep Q-Learning (DDQL) agent is deployed in each of the $N^{AP}$ APs of the Fog network. The agents in a given collaboration region are designed to manage the set of candidate Fog nodes $N^{F_C}$, i.e., a subset of Fog nodes from the larger network, in that region only. Hence, the size of state and action spaces can vary in different regions based on the number of candidate Fog nodes in each region, giving flexibility in forming those regions based on the needs of Fog service providers. This limits state and action spaces, and immediate rewards, to be collected from a smaller number of nearby Fog nodes, reducing the overhead of observing Fog queue information before every workload assignment decision. Hence, the state, action, and reward function of every agent in a given region, at every decision step $t$, are defined as follows:

\begin{itemize}
    \item \textbf{State:} $s_t = (w_t, \mathbb{Q}_{n, t} : \forall n \in N^{F_C})$, where $w$ is the number of compute instructions required by the workload to be assigned (representing its category) and $\mathbb{Q}_n$ is the number of queued jobs in each candidate Fog node $n$.
    \item \textbf{Action:} $a_t \in A$ is the decision to assign the current workload to a candidate Fog node from the set of candidate Fog nodes in the region of that agent.
    \item \textbf{Reward:} $r_{t} = - \mathbb{Q}_{a_{t-1}, t}$ is the number of queued jobs in the Fog node previously selected by that agent, i.e, the reward at decision step $t$ is the performance measure of the action taken at decision step $t-1$. Because the value of the reward is part of the observed state, it does not need to be transmitted separately to each agent. This value is then negated to represent the reward as a cost to be minimized instead of a utility to be maximized. The accumulation of these rewards is the expected discounted return $G = \sum_{t=0}^{\infty} \gamma^t r_t$, where $\gamma$ defines the importance of future rewards compared to recent ones. Giving importance to future rewards encourages the agents to learn a load distribution strategy that guarantees minimizing load accumulation in all Fog nodes over the long term. 
\end{itemize}

Our reward function enables the distribute agents to independently learn different load distribution strategies that keep the load in all Fog nodes minimal at all times. Each agent can discover its best strategy by considering its nearby Fog nodes before selecting remote Fog nodes in its region; this can be independently learned faster than learning a common policy for all agents. Hence, these independent agents can converge faster than single-agent solutions, as well as multi-agent solutions that lead to a common policy, i.e., explicitly and implicitly learn a common policy with CTDE and parameter/experience sharing, respectively. Also, learning independent policies allows us to reduce the state and action spaces by managing smaller subsets of Fog nodes, forming a collaboration region for each group of agents. Moreover, distributed solutions can provide near real-time decisions for every generated workload by avoiding the three-way communication overhead of centralized solutions (see Fig. \ref{fig:Topology_WF}\subref{fig:Topology_WF_C}). 

\begin{figure}[!htbp]
\centering
\subfloat[Centralized agent]{\includegraphics[width=0.48\textwidth]{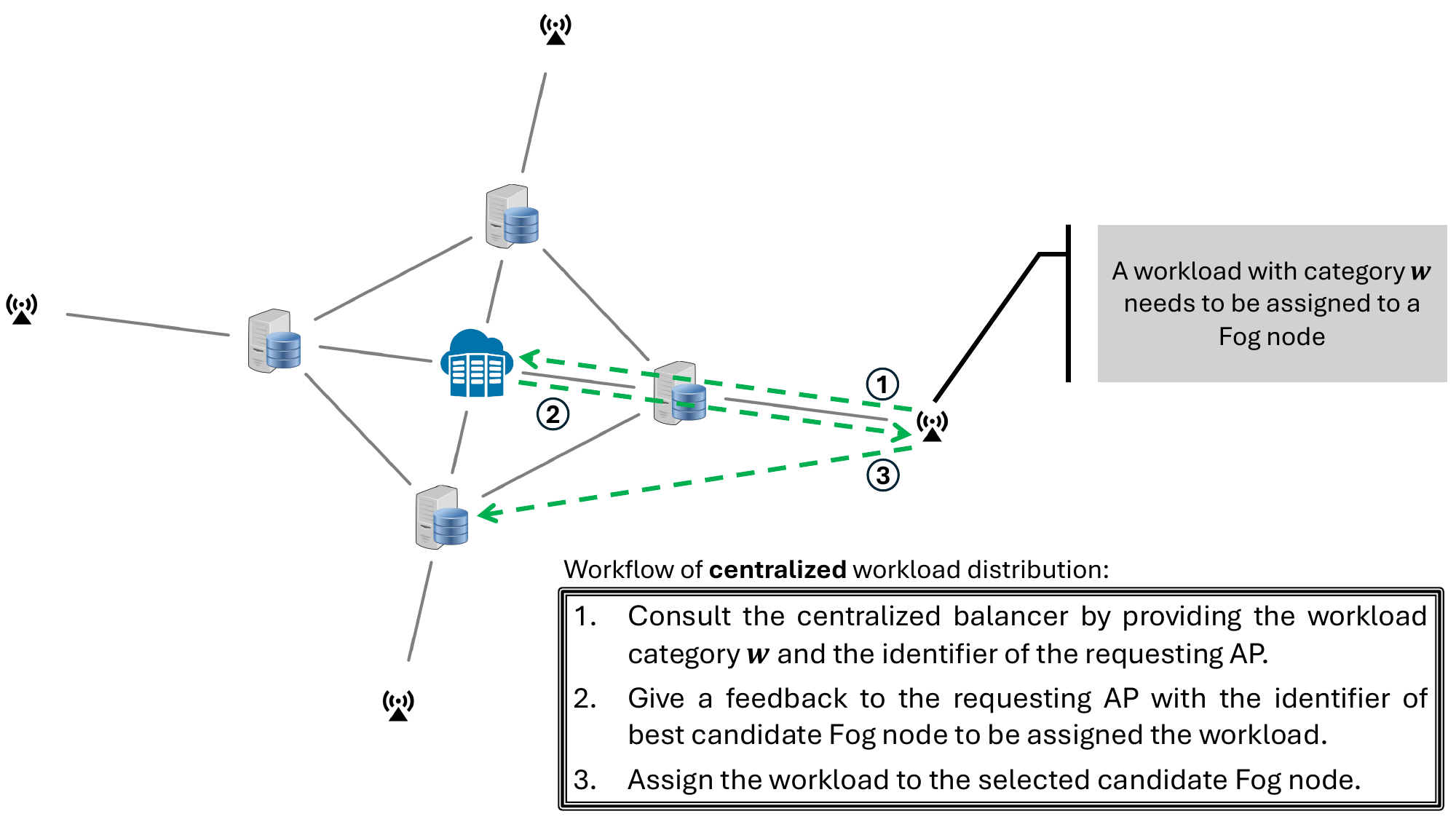}
\label{fig:Topology_WF_C}}
\hfil
\subfloat[Distributed agents]{\includegraphics[width=0.48\textwidth]{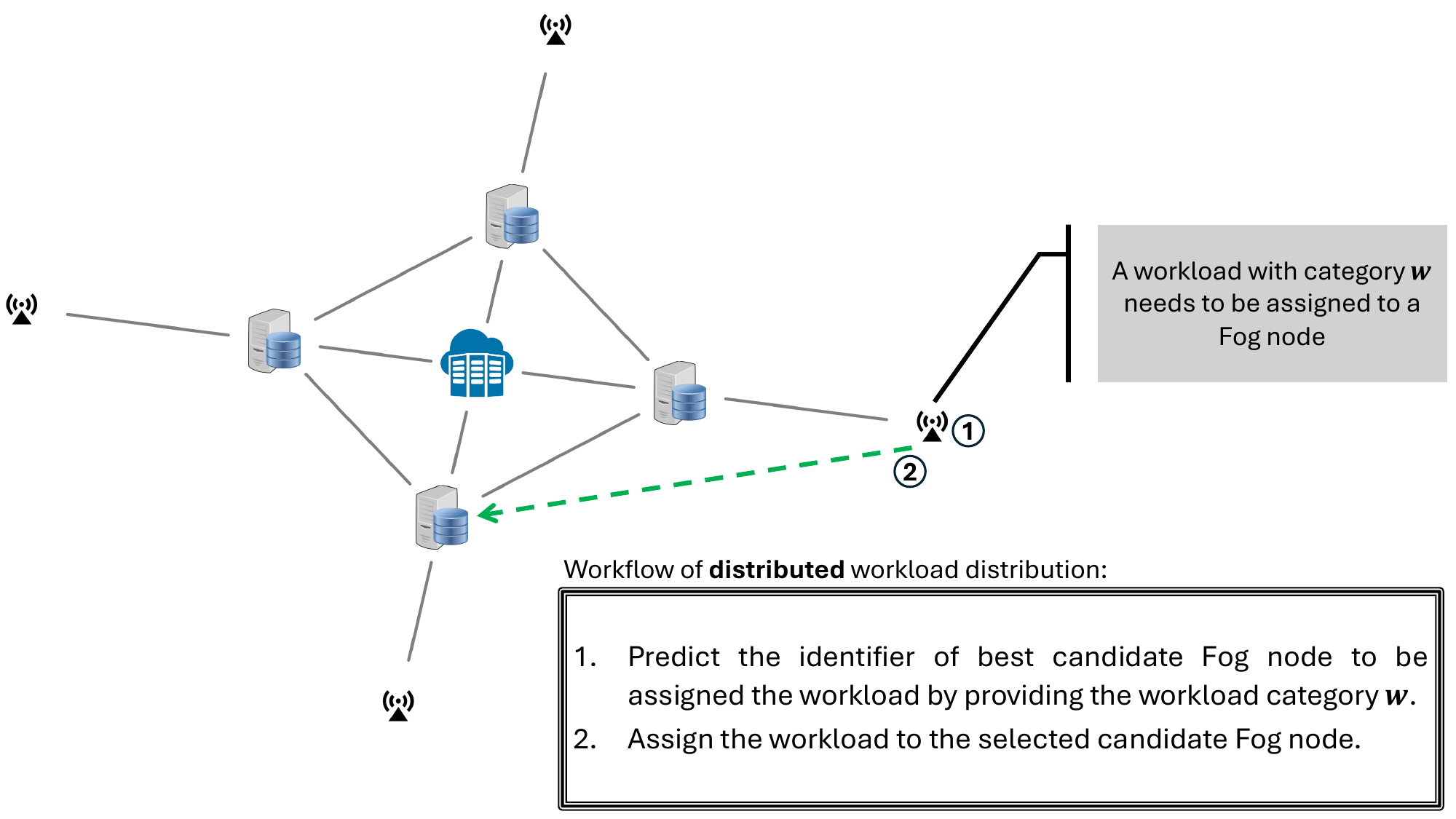}
\label{fig:Topology_WF_D}}
\caption{Workflow of workload distribution.}
\label{fig:Topology_WF}
\end{figure}

In centralized solutions, a request for every generated workload is sent to a distanced agent with the category of the workload and the identifier of its AP. The agent notifies the AP of the suitable Fog node to serve that workload based on workload category, AP identifier, and the load in each Fog node in the network. Only then, the AP can send the workload to the selected Fog node. With distributed solutions, however, agents make their own assignment decisions locally (see Fig. \ref{fig:Topology_WF}\subref{fig:Topology_WF_D}). The workload is then immediately sent to the selected Fog node, minimizing the decision delay compared to a centralized solution.

The time needed to observe Fog load information and the associated overhead of this process is often neglected in the literature \cite{realworld}. With a single-agent solution, all Fog nodes send their load information to that agent for every generated workload, as in Fig. \ref{fig:Topology_info}\subref{fig:Topology_info_C}. Although our distributed agents are grouped to manage smaller subsets of Fog nodes in proximity, this overhead is still significant since every agent needs to collect this information from every candidate Fog node in their region (see Fig. \ref{fig:Topology_info}\subref{fig:Topology_info_D}). This makes MARL observation communication overhead, for every decision, exponentially proportional to the number of Fog nodes and agents in each region.

\begin{figure}[!htbp]
\centering
\subfloat[Centralized agent]{\includegraphics[trim=50 30 50 40, clip, width=0.48\textwidth]{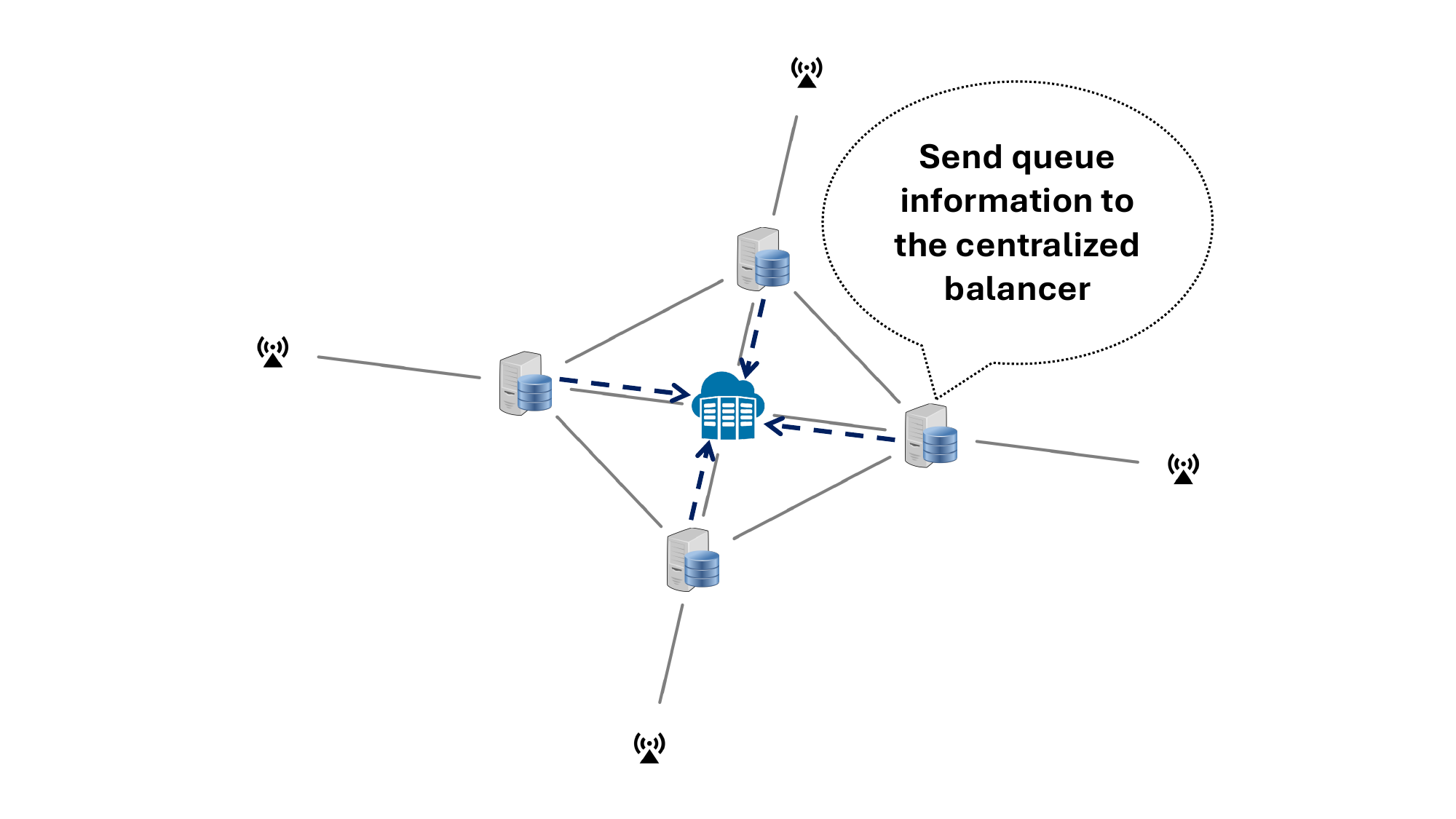}
\label{fig:Topology_info_C}}
\hfil
\subfloat[Distributed agents]{\includegraphics[trim=50 30 50 40, clip, width=0.48\textwidth]{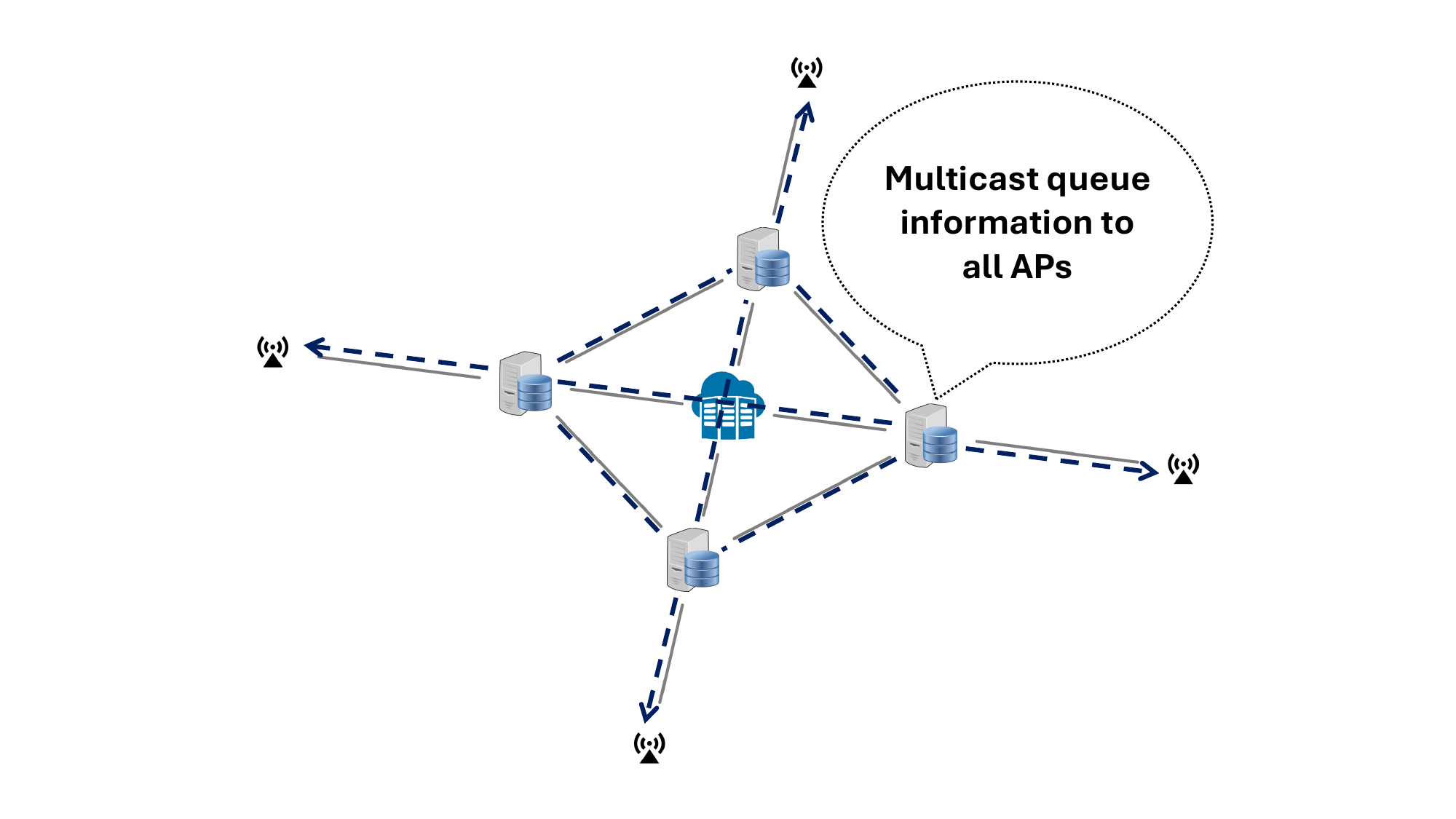}
\label{fig:Topology_info_D}}
\caption{Transmission of Fog queue information.}
\label{fig:Topology_info}
\end{figure}

Besides consuming bandwidth, observations can not be readily available before every decision due to communication delays, especially in large-scale networks with high generation rates. A more practical approach for deployment-ready solutions is to multi-cast such information at predefined intervals. This might lead to sub-optimal decisions due to using the same observation to make multiple assignment decisions, creating a non-stationary environment for the independent agents due to these outdated observations. This problem does not exist with centralized and coordination-based RL because the new environment state depends on a common policy behavior, not on the joint (unobserved) actions of independent agents. Being unrealistic to observe the real-time load on each Fog node before every decision, makes it more practical to consider a single observation for several actions within a given interval. Increasing observation frequency, using the smallest interval of existing multi-casting protocols (3 seconds for Gossip \cite{gossip}), reduces the impact of non-stationarity. 

In addition, dividing large-scale networks into smaller subsets of collaboration regions limits the transfer of Fog load information to fewer agents within a smaller geographical scope, thereby reducing the communication overhead. These regions can be designed based on the requirements of Fog service providers with the possibility of creating overlapped regions by sharing Fog resources between them. Managing fewer Fog nodes in proximity reduces state and action spaces for each agent and avoids the need to manage infeasible remote Fog nodes in large-scale networks. To ensure optimal load distribution in overlapped regions, the agents must be able to blindly learn independent load distribution strategies without knowing the policies and actions of the other agents. By independently learning to select the least loaded Fog node, our agents minimize the load on all Fog nodes in their regions, including shared Fog nodes.

\begin{lemma}
Given two non-overlapping sets of candidate Fog nodes $\mathcal{F}_1 \cap \mathcal{F}_2 = \emptyset$ in two collaboration regions (see Fig. \ref{fig:overlapping}). The resources in each region are optimized by a group of agents that blindly collaborate without knowing the actions of other agents in that region. If both regions share a subset of candidate Fog nodes $\mathcal{F}_1 \cap \mathcal{F}_2 = \mathcal{F}_S$, the agents of both regions can still blindly collaborate to optimize local and shared Fog resources.
\end{lemma}

\begin{proof}
Assuming the agents in an isolated region can blindly optimize their isolated Fog resources without knowing the actions of the other agents in their region, nor the mechanism of making those actions. Now, assume that a region with $\mathcal{F}_1$ candidate Fog nodes share $\mathcal{F}_S$ candidate Fog nodes with a region having $\mathcal{F}_2$ candidate Fog nodes (see Fig. \ref{fig:overlapping}). Suppose an agent in one of the regions does not collaborate with the other agents, leading to inefficient utilization of Fog resources in its region, including $\mathcal{F}_S$ resources. This implies that there is a breakdown in the collaborative optimization process. However, this contradicts our initial assumption that within each region, the agents can blindly optimize the resources of their candidate Fog nodes. Therefore, we conclude that the agents of both regions will collaborate to optimize both local and shared resources. Thus, the lemma holds.
\end{proof}

\begin{figure}[!htbp]
\centering
\includegraphics[trim=0 130 0 10, clip, width=.48\textwidth]{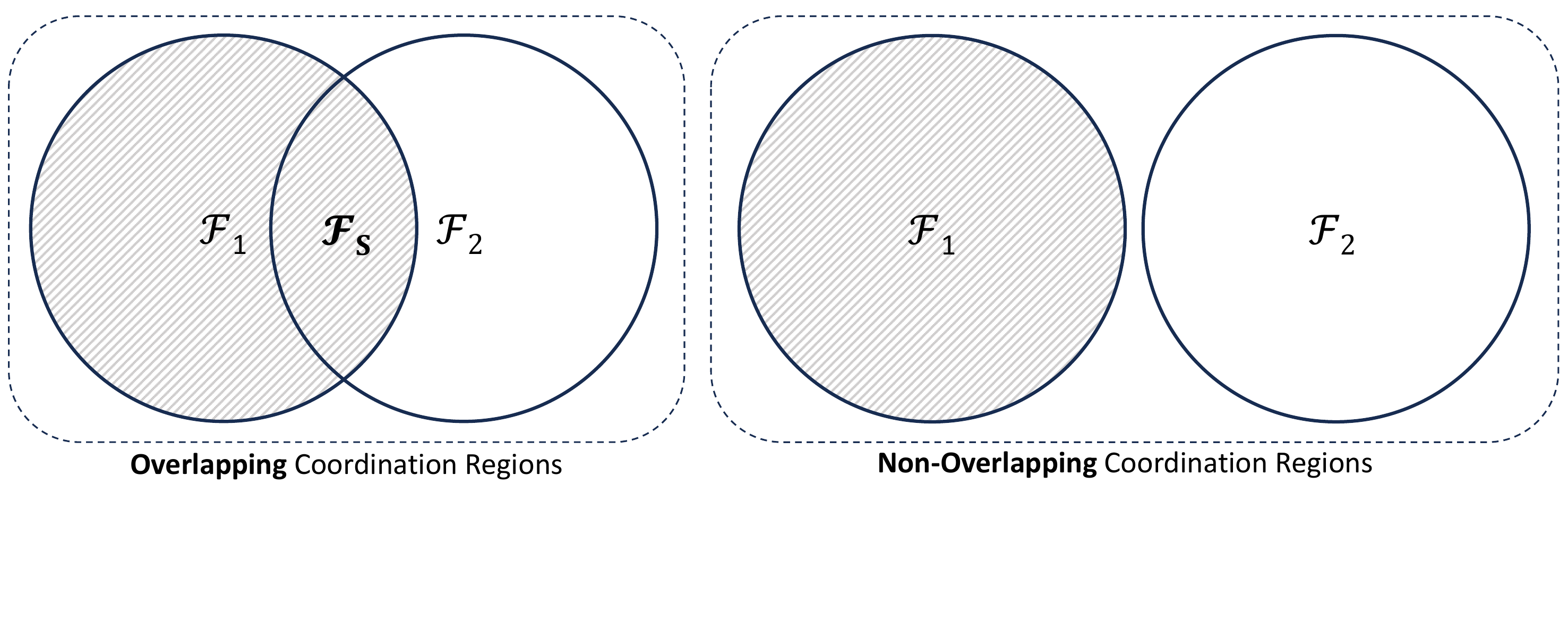}
\caption{Overlapping vs. non-overlapping sets of candidate Fog nodes for two coordinating regions.}
\label{fig:overlapping}
\end{figure}

This paper provides empirical proof supporting the ability of our agents to blindly learn independent load distribution strategies that will jointly optimize the Fog resources in their isolated region without any coordination between these agents. Validating the efficiency of our fully distributed solution under a realistic setup of one isolated region proves the solution's applicability in a global-scale multi-region network with isolated and overlapped collaboration regions. A fully distributed solution requires each agent to independently learn an efficient load distribution strategy, i.e., policy $\pi$, based on the category of the workload to be assigned and the current load in each Fog node in the region. Learning independent policies to optimize a common objective, i.e., selecting the least loaded Fog node, blindly encourages each agent to choose nearby Fog nodes before considering distanced ones. 

To learn the optimal policy $\pi^*$, the agents evaluate taking action $a$ in state $s$ following policy $\pi$ using the Q-function $Q_\pi(s, a) = \mathbb{E}\left[G \mid s, a, \pi \right]$. Therefore, the optimal Q-function $Q^*(s) = \max_\pi Q_\pi(s)$ cannot be less than the Q-function of any other policy, i.e., $Q_{\pi^*}(s) \geq Q_\pi(s)$. Deep Q-Learning (DQL) approximates the Q-function using a mapping $\phi$ of state-action pairs, allowing it to work in high-dimensional and continuous state spaces. This mapping is weighted by $\theta$ to calculate the Q-function $Q(s, a) = \theta \phi (s, a)$, which is updated according to Bellman equation: $Q(s, a) \leftarrow Q(s, a) + \alpha r + \gamma \max_{a'}Q(s', a') - \enspace Q(s, a)$.

The best action in a given state is the one that maximizes the Q-function for that state, i.e., $\pi(a\mid s) = \argmax_aQ(s, a)$. Once the algorithm converges to $Q^*$, the optimal action is predicted using the optimal policy $\pi^*(s) = \argmax_aQ^*(s, a) ~ \forall s \in S$. To avoid instability and divergence during training, a random set of interactions is batched from a Replay Buffer to train the Q-function network \cite{mnih2013playing}. In addition, an off-policy Double-DQL \cite{ddqn} balances exploration and exploitation by learning an optimal policy while using an exploratory policy for action selection. It uses two networks for the Q-function, one updates every training step to evaluate the actions, and the other updates periodically for stable action selection.

To mitigate the computational cost of continuous training in a practical deployment, each agent adopts our lifelong learning framework previously proposed in \cite{ebrahim2023lifelong} (see Fig. \ref{fig:lifelong}). This framework can be used to initialize new agents in a simulation, or by leveraging the expertise of existing agents in the network, before fine-tuning them in their intended deployment environment, mitigating the impact of bad actions when training from scratch in real networks. With this framework, the agents can stop training when they converge to the optimal policy (performance saturation) and only retrain in case of significant performance degradation due to environmental changes. When convergence is detected, i.e., when the average reward does not improve further, the agents extract lightweight inference models to provide inference using only simple arithmetic and array operations.

\begin{figure*}[!htbp]
\centering
\includegraphics[trim=0 90 0 50, clip, width=\textwidth]{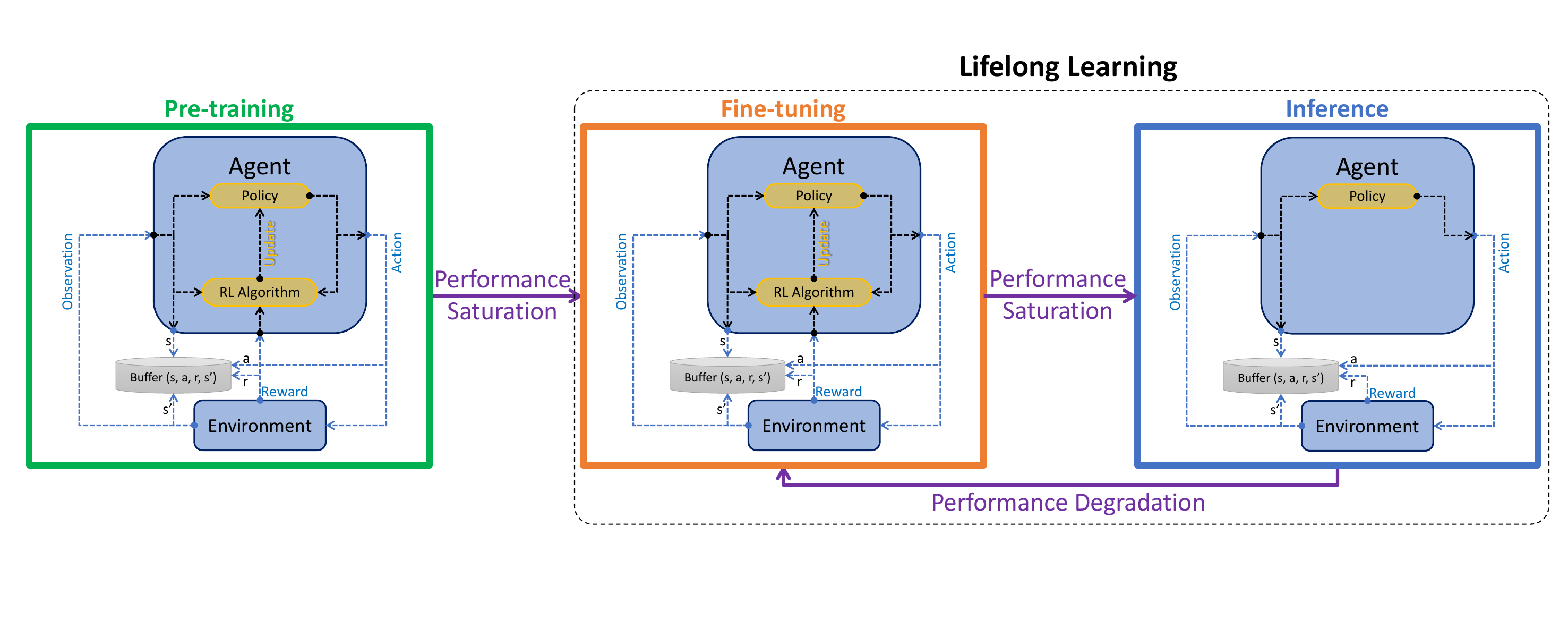}
\caption{Lifelong learning for RL agents.}
\label{fig:lifelong}
\end{figure*}

Changes in the environment can then be detected by observing a significant degradation in the average reward, i.e., an increase in the average load accumulation in the region, triggering the need for agents to perform a corrective measure to adapt to those changes. The agents are then fine-tuned using transfer learning to speed up convergence, improve data efficiency, and provide better generalization than learning from scratch \cite{taylor09}. Hence, policy network parameters are transferred from source to target task ($\theta_\mathcal{T}\leftarrow\theta_\mathcal{S}$) along with $x$ recent replay buffer interactions ($\mathcal{B}_\mathcal{T}\leftarrow \text{RecentExperiences}(\mathcal{B}_\mathcal{S}, x)$), providing fast and efficient lifelong adaptation to changes \cite{ebrahim2023lifelong}. Since effective transfer learning requires source and target tasks to share a certain level of similarity \cite{ebrahim2019will}, we introduce changes only through significant surges in workload generation rates. Using transfer learning allows the agents to learn robust policies using information $\mathcal{I}$ from both source and target environments, i.e., $\pi_\theta=\phi(\mathcal{I}_\mathcal{S}{\sim}E_\mathcal{S}, \mathcal{I}_\mathcal{T}{\sim}E_\mathcal{T}): S^\mathcal{T}\rightarrow A^\mathcal{T}$. 

\section{Performance Evaluation}
\label{sec:eval}

This paper compares learning independent decision-making policies with learning a common policy, where a common policy results from single-agent RL, CTDE MARL, or even distributed MARL training with parameter or experience sharing. Hence, we compare the performance of learning independent policies by our proposed fully distributed MARL approach against learning a common policy by a single agent as in \cite{ebrahim2023lifelong}. This comparison will highlight the advantages of training independent agents compared to a single LB agent, especially with convergence speed, performance, and scalability. To show why learning independent policies is faster than learning a common policy, Fig. \ref{fig:resources} shows a scenario where considering the remote rich resources of Fog node $F_2$ is beneficial to the agent in $AP_1$ while considering the remote limited resources of Fog node $F_1$ is bad for the agent in $AP_2$. Both agents should learn to use nearby Fog resources while carefully considering only feasible remote Fog resources. Using a common policy to learn a different strategy for each $AP$ is harder, more complex, and requires more training than learning an independent policy for each $AP$. Learning a policy for each agent independently allows every agent to find its optimal load distribution strategy faster by avoiding the complexity of learning a global policy that satisfies the needs of all $APs$.

\begin{figure}[!htbp]
\centering
\includegraphics[trim=150 30 150 30, clip, width=0.48\textwidth]{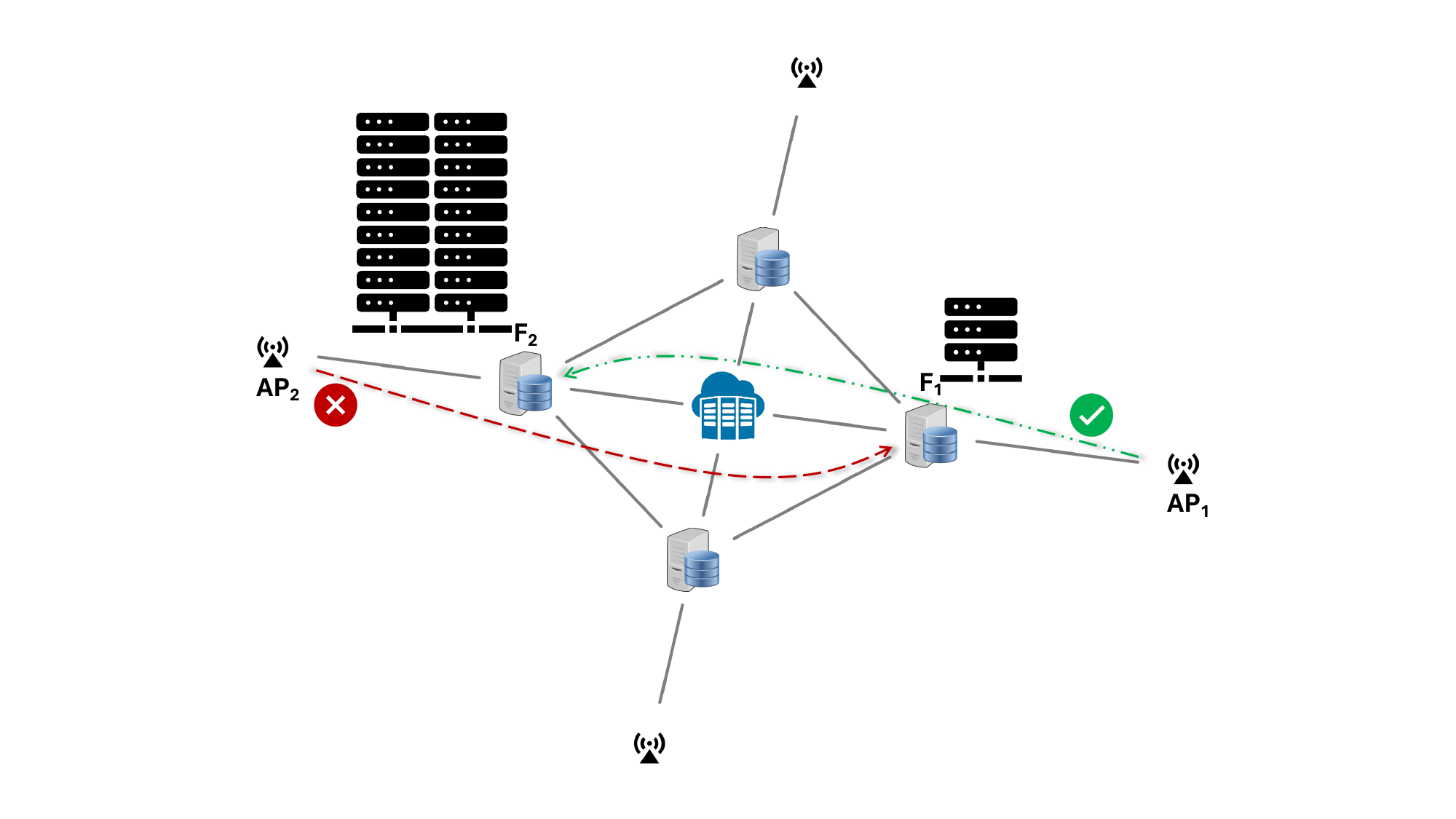}
\caption{Learning a common policy vs. independently learned policies.}
\label{fig:resources}
\end{figure}

The system parameters of the experiments in this work are based on the experiments in our previous contributions \cite{EBRAHIM2023resilience, EBRAHIM2023privacy, ebrahim2023lifelong}, ensuring consistency and relevance in our evaluation. We used TF-Agents \cite{TFAgents} to implement the DDQL algorithm with the hyper-parameters in Table \ref{tab:paramsddqn}. A fully independent agent is deployed in each AP with its policy model, Q-Network, and experience replay buffer without sharing any information between the agents. The simulations run on a DELL G5 laptop with 16GB RAM and Intel Core i7-9750H (6 physical cores, 12MB Cache, running up to 4.5 GHz). The agents are trained using NVIDIA GeForce RTX 2070 GPU with Max-Q Design (8GB RAM). To simulate realistic Fog environments we used YAFS Discrete-Event Simulator \cite{YAFS}.

\begin{table}[tb]
\small
\caption{DDQL Hyper-parameters \label{tab:paramsddqn}}
\centering
\begin{tabular}{|l|l|}
\hline
\textbf{Parameter} & \textbf{Values}\\
\hline
Discount Factor $\gamma$ & 0.99 \\
\hline
Decayed Epsilon-Greedy & The first 75\% Training steps \\
\hline
Decayed Epsilon values & Linearly from 100\% to 1\% \\
\hline
Replay Buffer Type & Uniform \\
\hline
Maximum Buffer Capacity & 1 Million trajectories of experience \\
\hline
Initial Buffer Population & 10\% of its maximum capacity \\
\hline
Buffer mini-Batch Size & 50 samples of 2 steps each \\
\hline
Network Train Period & 4 Decision steps \\
\hline
Target Update Period & 2000 Decision Steps \\
\hline
Network Layers & Fully connected [256, 128, 64] \\
\hline
Network Optimizer & Adam with $2.5e^{-4}$ learning rate \\
\hline
Network Loss Function & Huber Loss \\
\hline
\end{tabular}
\end{table}

Our distributed agents are evaluated on isolated Fog regions randomly generated with complex flat topologies, i.e., multiple interconnected Fog nodes. To realistically simulate Fog networks, each Fog region is created, using the Networkx library \cite{networkx}, from random undirected graphs resembling Internet Autonomous System networks. The nodes in these graphs are identified as the Cloud, Fog nodes, or IoT APs based on their betweenness centrality \cite{betweenness}, i.e., the Cloud is in the center of the graph while APs are at the edges (see Fig. \ref{fig:topo}). To resemble unbalanced load distribution relative to Fog resources, Fog nodes are assigned heterogeneous compute resources (uniformly between $[10^3,10^5]$ Instructions Per Second (IPS)) inversely proportional to the number of their directly connected APs. 

\begin{figure}[!htbp]
\centering
\includegraphics[trim=20 15 20 20, clip, width=0.48\textwidth]{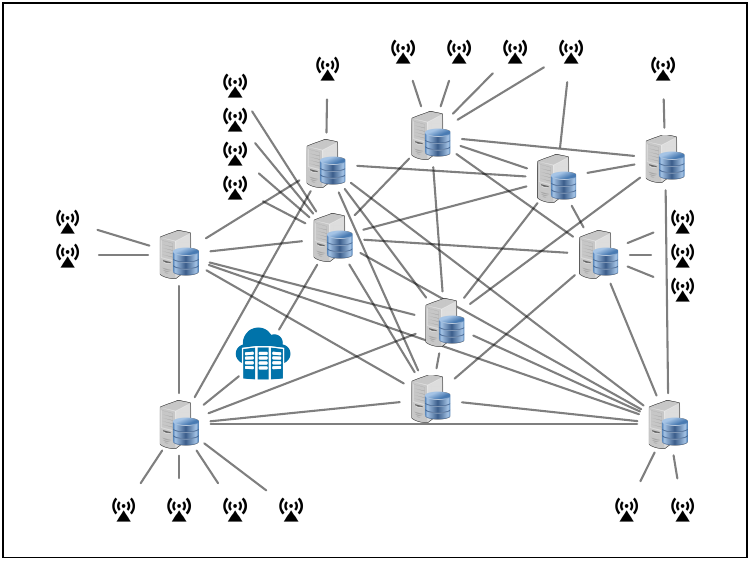}
\caption{A Fog region with unbalanced flat topology with 11 interconnected Fog nodes and 21 IoT APs.}
\label{fig:topo}
\end{figure}

The Cloud is connected to two Fog nodes with the highest betweenness centrality and has computing resources of $10^6$ IPS. To simulate undedicated public networks with a distanced Cloud and Fog nodes towards the edge of the network, the propagation delays of IoT-Fog, Fog-Fog, and Fog-Cloud links are drawn from uniform distributions in the ranges $[1,2)$, $[2,4)$, and $[10,20)$ seconds, respectively \cite{FLoadNet}. The bandwidths of those link categories are drawn from uniform distributions in the ranges $[10^2,10^3)$, $[10^3,10^4)$, and $[10^3,10^4)$ Mbps, respectively \cite{FLoadNet}. To create a system with scarce resources, workloads with heavy, moderate, and light demands are simultaneously generated with compute requirements ($10^4$, $10^3$, and $10^2$ instructions, respectively) and data sizes ($10^3$, $10^2$, and $10^1$ Bytes, respectively) relative to Fog and link resources.

These workloads are generated as a Poisson Point Process using exponential distributions to effectively model the stochastic nature of job arrivals in Fog computing environments \cite{MCDM}. We used three scales for the exponential distributions, i.e., $\beta=$ 200, 150, and 100, to evaluate the performance under low, medium, and high load intensities, respectively. Without efficient load distribution, heavy workloads with high generation rates can easily create system bottlenecks due to the unbalanced Fog environment. This makes learning optimal load distribution strategies with high generation rates much harder than learning those strategies in environments with less demand. Therefore, our agents are trained with a low rate first for 30K training steps (in episodes of 10K simulation steps) before the rate instantly increases to the next load intensity. Using transfer learning, the agents rapidly adapt to the increased generation rates as they intensify from low to medium, and then from medium to high.

In the following section, we compare the proposed fully Distributed RL (DRL) approach against a Centralized single-agent RL (CRL) solution to underscore the benefits of independent decision-making over learning a common policy. To demonstrate the effectiveness of intelligent load distribution compared to simple selection methods, we also included \textbf{Random} and Distributed Round-Robin (\textbf{DRR}) selection, as well as selecting the \textbf{Nearest} Fog node to each AP and the \textbf{Fastest} Fog service for each workload. The average system waiting delay is the performance metric that allows us to evaluate the ability to reduce workload accumulation in each Fog node. We also plot the variations in waiting times and resource utilization between all Fog nodes in the region, which helps evaluate the agents' ability to balance the load in the network and provide fairness in resource utilization, respectively. We then measure the average end-to-end execution delay for the IoT application loop to analyze the impact of network latency and service times on load distribution decisions based solely on reducing waiting delay.

\subsection{Performance Analysis}
\label{subsec:results}

The results in this section are averaged over 10 experimental trials, each with a different random seed, to illustrate the confidence intervals of performance across random variations in Fog architecture. These confidence intervals capture the variability in results due to topology changes, aiding in assessing our findings' precision, reliability, and robustness across different environments. These results represent the performance in an evaluation episode that is $10\times$ longer than training episodes, i.e., 100K simulation steps, offering a measure of generalization for lifelong inference during deployment. The workload generation rate in this evaluation episode is high to evaluate load distribution with frequently generated workloads, which is very difficult to learn \cite{ebrahim2023lifelong}. When compared against baseline methods, a logarithmic scale is used on the results to show the performance of RL-based solutions due to their significantly superior performance.

Fig. \ref{fig:avgWaiting}\subref{fig:avgWaiting1} illustrates the mean system waiting delay achieved using each load distribution method, which is the average time a workload must wait in the queue of a Fog node before being served. The results show superior performance with RL-based methods compared to simple baselines, highlighting the need for intelligent load distribution in unbalanced Fog environments. The results also show notable performance gains when transitioning from centralized to distributed RL, achieved by learning independent load distribution strategies that leverage resources in proximity differently for each distributed agent. Blindly learning separate policies by independent agents not only outperforms learning a common policy for all APs, but also eliminates the need to share network parameters, buffered interactions, or the actions of other agents. The small error bars in the figure show the robustness of RL-based solutions against variations in Fog topology across multiple experimental runs, showing the learning ability of the agents despite changes in problem complexity.

\begin{figure}[!htbp]
\centering
\subfloat[DRL vs. CRL and baselines]{\includegraphics[width=0.47\textwidth]{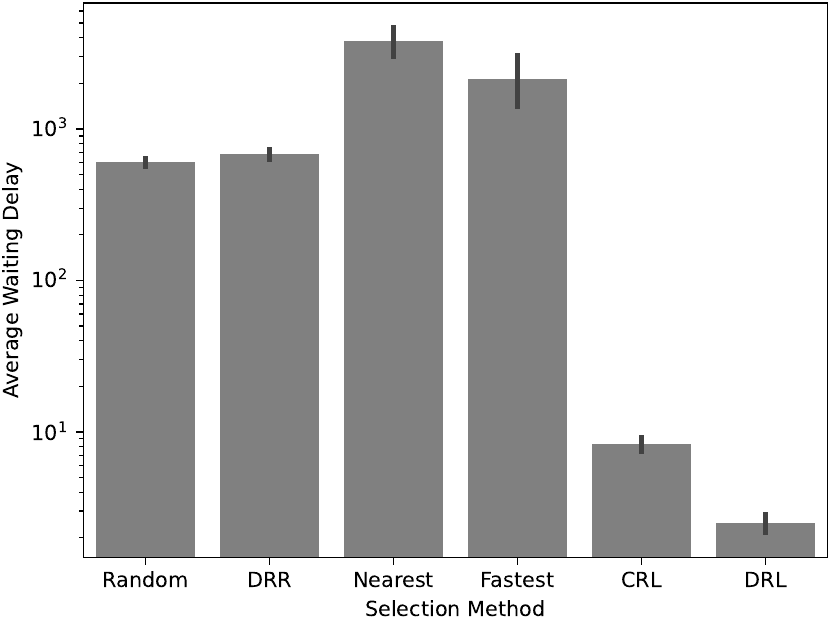}
\label{fig:avgWaiting1}}
\hfil
\subfloat[Realistic vs. unrealistic observations]{\includegraphics[width=0.47\textwidth]{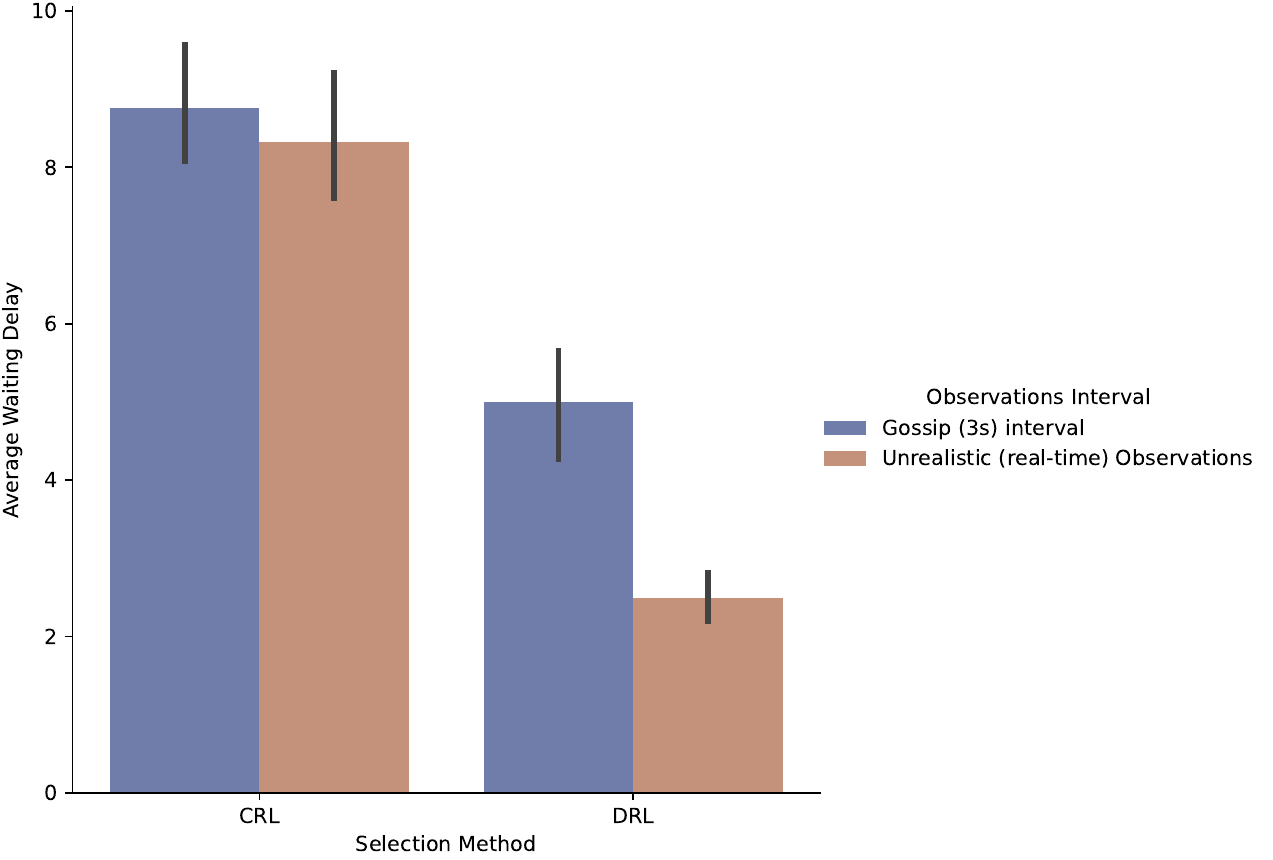}
\label{fig:avgWaiting2}}
\caption{The average system waiting delay.}
\label{fig:avgWaiting}
\end{figure}

Fig. \ref{fig:avgWaiting}\subref{fig:avgWaiting2} compares DRL with CRL only focusing on the impact of practical interval-based observations (observing Fog load every 3 seconds) compared to unrealistic real-time observations (a new observation for every action). This comparison is essential to show the trade-off between theoretical performance and realism of practical deployment constraints. This trade-off is often neglected in the literature by assuming Fog load information (from across the network/region) is readily available before every action. The figure shows that using the same observation to make multiple assignment decisions has a larger impact on DRL than CRL. The slightly inferior performance of DRL with interval-based observations, compared to the unrealistic optimal solution where observations are unrealistically provided in real-time, is due to the non-stationarity problem created by outdated information observed by independent agents. The impact on CRL with interval-based observations is smaller than DRL because the non-stationarity does not exist with CRL as the new observation depends on the behavior of a single agent.

Fig. \ref{fig:stdWaiting}\subref{fig:stdWaiting1} shows the variations in waiting delays between Fog nodes, measured using the standard deviation of waiting times across all Fog nodes. Lower values represent more equalized or balanced waiting times between Fog nodes. The error bars in the figure indicate the variability in these values across different experiments, demonstrating the level of result consistency under various Fog topologies. The results show how intelligent RL-based solutions can better balance the waiting times across Fog nodes than simple baselines, proving their ability to learn efficient load distribution strategies that minimize workload accumulation in all Fog nodes. By distributing RL decision-making, each agent can independently learn to distribute its load by leveraging nearby resources based on its location in the network. Therefore, DRL can balance the waiting times better than CRL since independently learning effective load distribution strategies, that fit the needs of each agent, is faster than learning a common policy for all APs. 

\begin{figure}[!htbp]
\centering
\subfloat[DRL vs. CRL and baselines]{\includegraphics[width=0.47\textwidth]{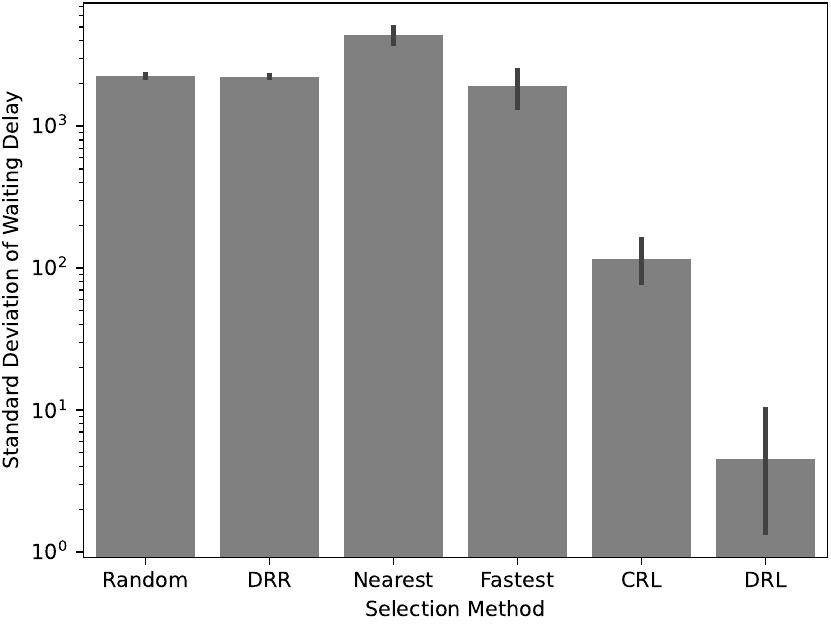}
\label{fig:stdWaiting1}}
\hfil
\subfloat[Realistic vs. unrealistic observations]{\includegraphics[width=0.47\textwidth]{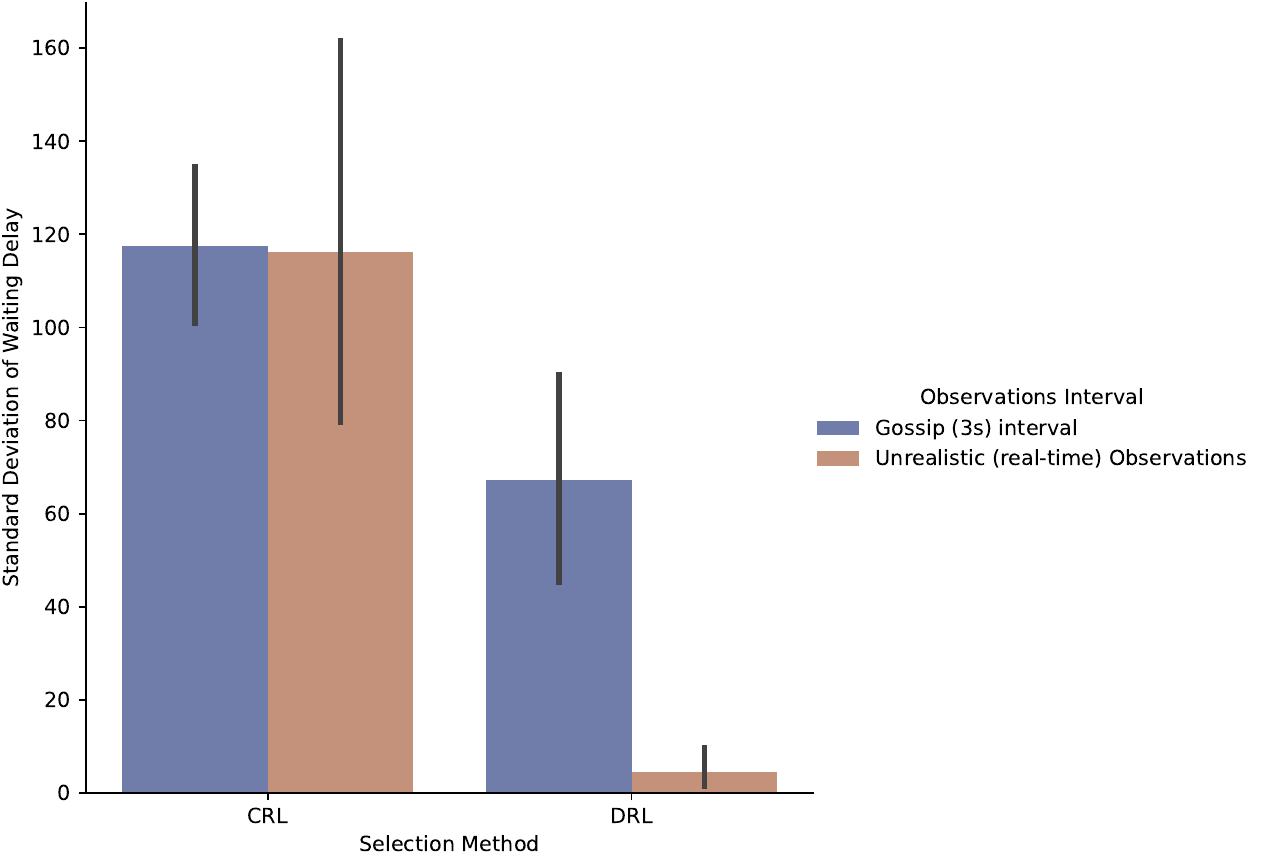}
\label{fig:stdWaiting2}}
\caption{The variations in the waiting delay between Fog nodes.}
\label{fig:stdWaiting}
\end{figure}

To focus on the impact of interval-based observations on balancing the waiting times between Fog nodes, we again limit our comparison in Fig. \ref{fig:stdWaiting}\subref{fig:stdWaiting2} to DRL and CRL only. The figure confirms that due to interval-based observations, outdated observations have a negligible impact on CRL compared to DRL. DRL is more affected by outdated observations because of the environment's non-stationarity caused by the unobserved actions of multiple independent agents. With CRL, a single agent alters the state of the environment, making it stationary from that agent's perspective, even if the same observation is used for multiple actions. However, DRL's ability to balance waiting times between Fog nodes remains significantly better than CRL's, with real-time and interval-based observations. With real-time observations, DRL demonstrates an impressive ability to achieve theoretically optimal load distribution, resulting in almost identical waiting times across all Fog nodes in the environment. 

Besides balancing waiting times between Fog nodes, RL-based solutions balance the consumption of Fog resources by reducing the variation in resource utilization between Fog nodes (see Fig. \ref{fig:stdUtilization}\subref{fig:stdUtilization1}). Balancing resource utilization provides fairness across Fog nodes and mitigates the probability of overloading or underloading individual Fog nodes in the region. This enhances system performance by encouraging all Fog nodes to equally contribute with their resources relative to their capacity and constraints. It also improves system stability by ensuring the maximum availability of spare Fog resources in the network, minimizing workload-dropping probability in case of a sudden surge in generation rates. While both RL methods are better than the baselines, DRL is significantly better in providing fair resource utilization using independently learned load distribution strategies. This proves the ability of our fully distributed agents to balance the system load without requiring the exchange of coordination information.

\begin{figure}[!htbp]
\centering
\subfloat[DRL vs. CRL and baselines]{\includegraphics[width=0.47\textwidth]{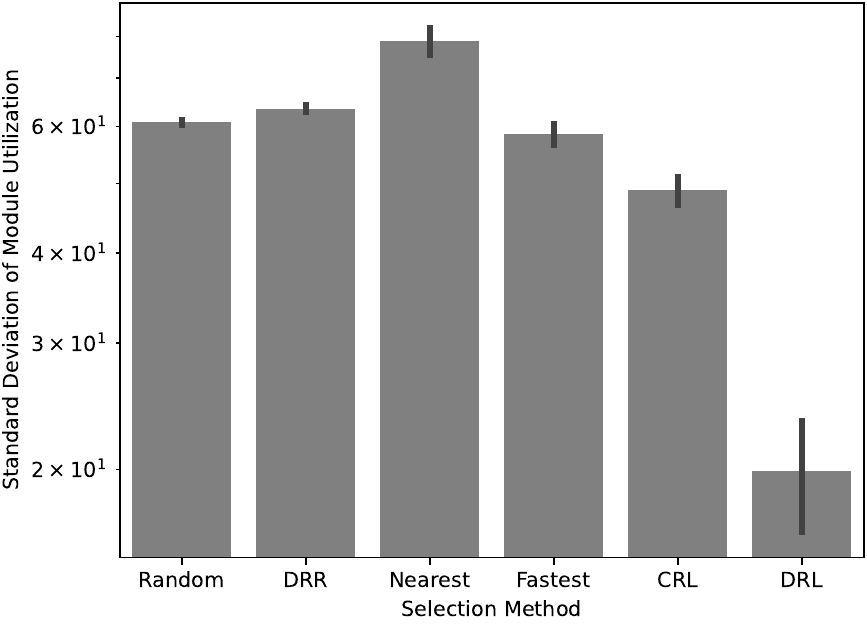}
\label{fig:stdUtilization1}}
\hfil
\subfloat[Realistic vs. unrealistic observations]{\includegraphics[width=0.47\textwidth]{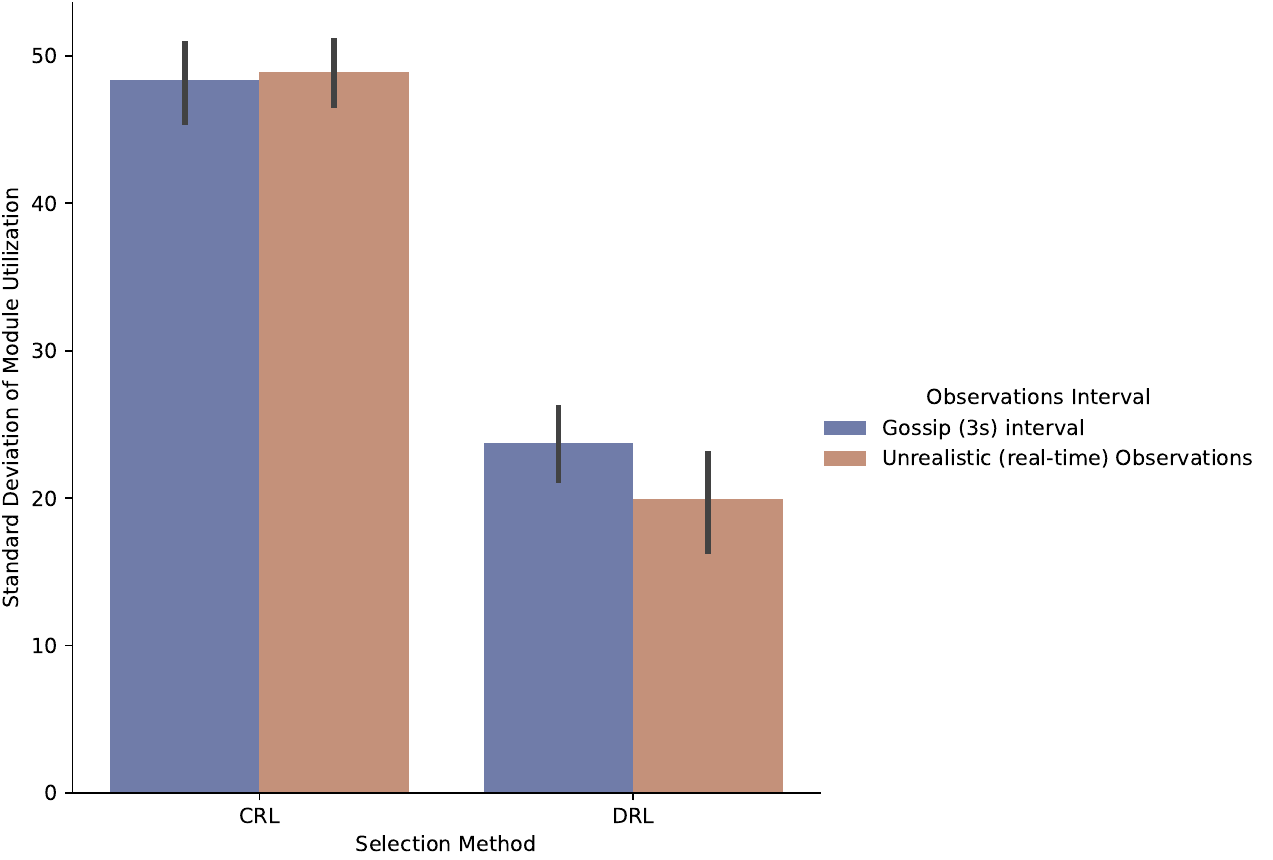}
\label{fig:stdUtilization2}}
\caption{The variations in the module utilization between Fog nodes.}
\label{fig:stdUtilization}
\end{figure}

Once again, we narrow down the comparison to DRL and CRL only to study the impact of interval-based observations on balancing resource utilization between Fog nodes (see Fig. \ref{fig:stdUtilization}\subref{fig:stdUtilization2}). The figure confirms our previous findings that the choice of observation frequency, i.e., real-time and interval-based observations, does not impact the performance of CRL due to having a single decision-making entity. Whether the CRL agent observes the environment before every action or makes an observation for all actions within 3 seconds, the variation in resource utilization (fairness) between Fog nodes is almost identical. DRL, however, is slightly impacted when observing the environment less frequently compared to unrealistically observing real-time Fog load before every action. Observing the load from Fog nodes once every 3 seconds is more practical for the actual implementation and deployment of the solution, and a slight degradation in the achievable performance can be tolerated given the constraints of real-world networks.

To analyze the impact of network latency and service times on load distribution decisions based solely on reducing waiting delay, we need to show the average end-to-end execution delay for the whole IoT application loop. This will demonstrate the direct relationship between minimizing the waiting delay of individual workloads and minimizing the execution delay of the flow of workloads that make IoT applications. Fig. \ref{fig:execution}\subref{fig:execution1} shows the ability of RL-based solutions to significantly minimize, compared to the simple baselines, the average execution time of the three application loops simultaneously running in the system. DRL performance is still slightly better than CRL due to the independent agents that learn different load distribution strategies based on their relative locations to Fog nodes, reducing workload accumulation in all Fog nodes better than using a single policy for the whole region. These independent agents find the optimal solution faster than learning a common policy by the single-agent solution. By avoiding the exchange of coordination information, and its associated overhead, these independent agents will also converge faster than implicitly learning a common policy using distributed agents that share their network parameters, experiences, or actions. 

\begin{figure}[!htbp]
\centering
\subfloat[DRL vs. CRL and baselines]{\includegraphics[width=0.47\textwidth]{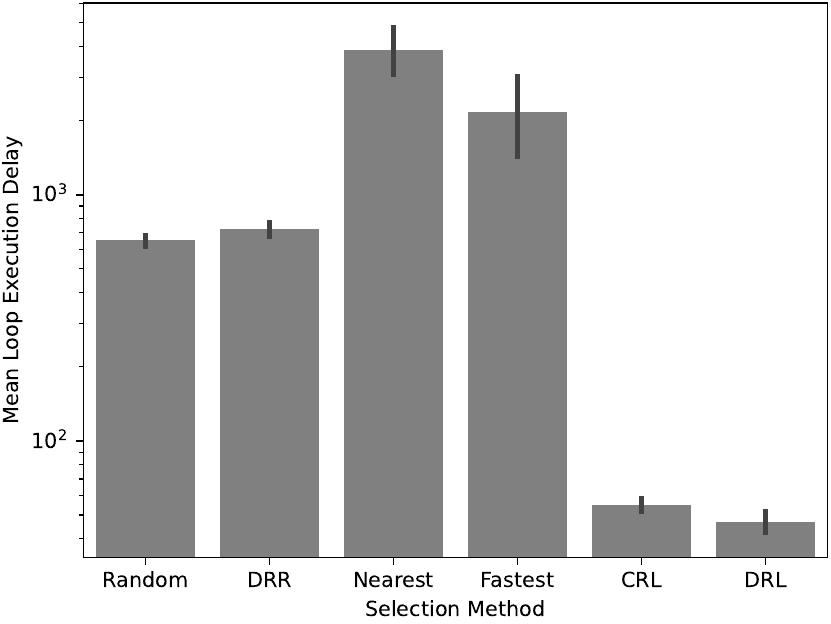}
\label{fig:execution1}}
\hfil
\subfloat[Realistic vs. unrealistic observations]{\includegraphics[width=0.47\textwidth]{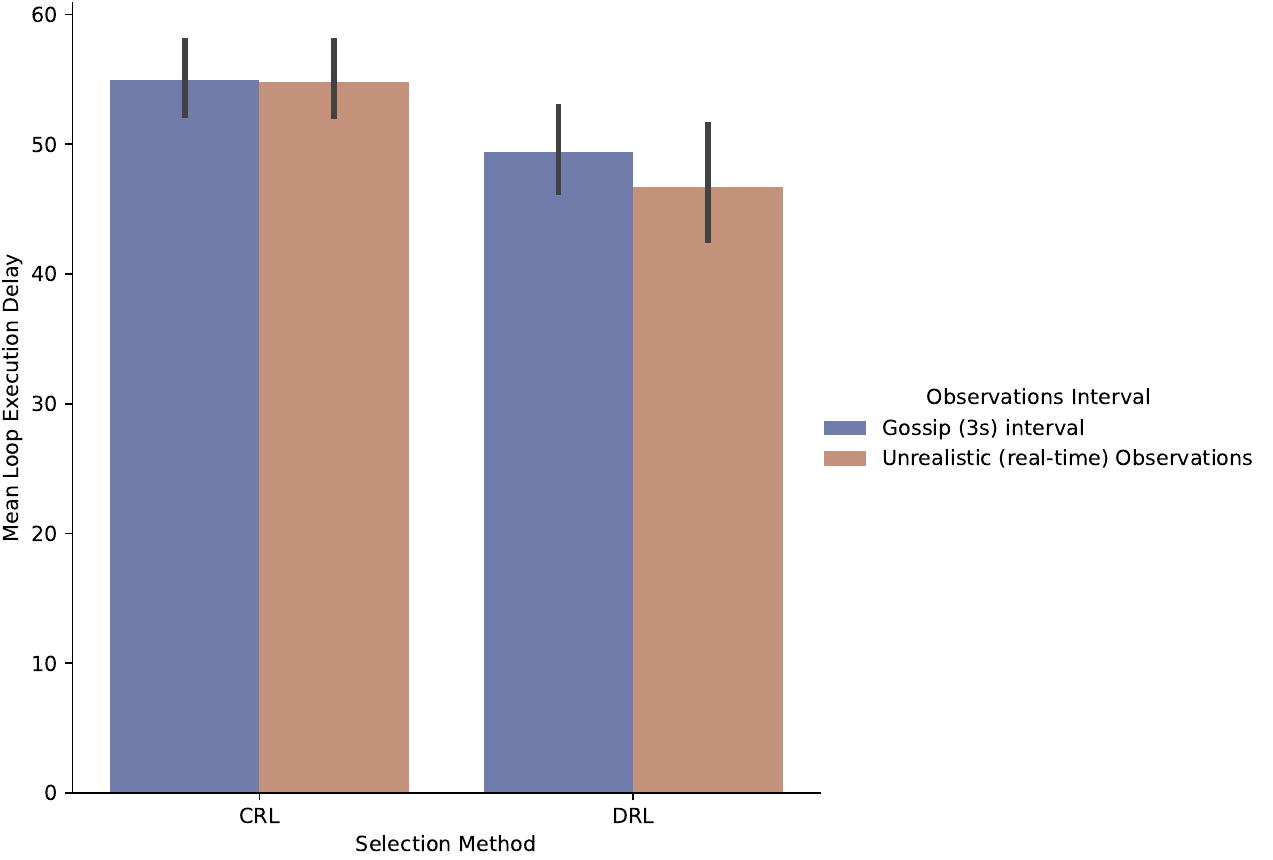}
\label{fig:execution2}}
\caption{The average total execution delay.}
\label{fig:execution}
\end{figure}

With Fig. \ref{fig:execution}\subref{fig:execution2}, we can analyze the impact of interval-based observations on the average end-to-end execution delay of IoT application loops. Once again, the results show that real-time and interval-based observations lead to almost identical average execution delays using CRL. Even with outdated observations, the series of actions, by a single agent, between two consecutive observations determine the next observed state of the environment, leading to environment stationarity. Although DRL achieves a smaller average execution delay than CRL, DRL is slightly affected by these outdated observations due to having independent actors that simultaneously change the state of the environment. Even with interval-based observations, the performance of DRL is still better than CRL. The design of smaller regions can reduce DRL's non-stationarity problem by minimizing the maximum communication delay for multi-casting Fog load information. This reduction enables the selection of smaller, yet practical, observation intervals based on the dimensionality of the region, lowering the impact of interval-based observations by independent decision makers.

In summary, the results presented in this section demonstrate the superior performance of our proposed fully distributed MARL solution (DRL) compared to the CRL approach (decision-making using a common policy) and other baselines. The results highlight the ability of independently trained agents to learn distinct load distribution strategies more quickly than when learning a common policy. These strategies collectively reduce workload accumulation in Fog nodes more effectively than a CRL policy with the same amount of training. The fully distributed agents effectively minimize average waiting delay, balance waiting times, and ensure fair resource utilization across all Fog nodes in their region, thereby improving the average end-to-end execution delay. 

Without centralized or inter-agent coordination, our fully distributed agents can jointly learn optimal load distribution strategies more quickly and with less overhead than learning a common policy. A common policy is typically learned explicitly by a single agent or CTDE solutions, or implicitly through parameter, experience, or action sharing. Also, learning independent policies that optimize a common objective enables the management of smaller collaboration regions within a global network, while still ensuring global optimality. Using small collaboration regions allows for a scalable solution that reduces problem complexity since the agents have fewer Fog nodes to manage, i.e., state and action spaces are smaller, making learning easier and converging faster. The results also emphasize the trade-off between performance and real-world applicability, as we assess the impact of interval-based observations rather than assuming idealized conditions where Fog load information is continuously available before every action.

\section{Conclusion}
\label{sec:conclusion}

This paper presents a scalable Fog load balancing solution using fully distributed RL agents that independently learn to reduce workload accumulation in Fog nodes without exchanging learning or behavior information. Blindly optimizing the common objective of reducing Fog load, with independent strategies, allows our agents to reduce the average waiting delay, balance waiting times between Fog nodes, and provide fair resource utilization, leading to optimized end-to-end execution delay for IoT applications. Deploying an independent agent for each IoT AP leads to separate load distribution strategies, i.e., policies, that converge faster than learning a common policy for all APs. With limited training, our results show that learning independent policies outperforms explicitly learning a common policy by a single agent or CTDE-based distributed agents. Even with separate training, coordination between MARL agents implicitly leads to a common policy due to sharing network parameters, buffer experiences, or load distribution decisions.

Unlike centralized and coordination-based load-balancing solutions, fully distributed agents allow global-scale implementations. The independent agents can be grouped to manage smaller sets of Fog nodes from the global network, forming Fog regions that might overlap by shared Fog resources. The ability of these agents to optimize a common objective in an isolated region without coordination guarantees their ability to optimize that objective even when some Fog resources are shared with another region. Managing fewer Fog nodes in each region reduces the state and action spaces, which makes learning easier and faster. To reduce required training further, the agents use a lifelong learning framework to adapt faster to environmental changes while mitigating the computational cost of continuous training. To do this, the agents are fine-tuned using transfer learning as generation rates significantly increase in the environment. Transferring policy network weights and replay buffer experiences between sources and target learning tasks leads to faster convergences and better generalization than learning from scratch. 

Each agent needs load information, i.e., for observations and rewards, from every Fog node in the network to make optimal load distribution decisions. Given communication latency, we can not assume the real-time availability of such information before every decision, especially with frequently generated workloads in large networks. Unfortunately, this assumption is common in the literature to report optimal performance that is not achievable in practice. Because the cost of collecting this information has never been considered before, we analyze a practical approach to observing Fog information in practical deployment. This analysis shows the trade-off between theoretical and practical performance by evaluating interval-based observations with realistic intervals from the Gossip multi-casting protocol. The findings demonstrate no impact on the single-agent solution and a slight deterioration in performance on the distributed agents. The unavoidable performance deterioration by independent agents is due to outdated observations, which make the environment non-stationary from the point of view of each agent, leading to harder learning. 

For future work, adaptive observation intervals could be explored as an alternative to fixed intervals, where the intervals are adjusted based on demand fluctuations. This approach could reduce overhead and mitigate the impact of outdated observations. Moreover, observation intervals could vary for each agent based on its location in the network relative to the Fog nodes it manages. Additionally, a real-world implementation is essential to validate the scalability and effectiveness of the proposed fully distributed load-balancing solution in practice. Deploying independent agents across geographically dispersed Fog environments with heterogeneous resources would provide a more accurate and realistic evaluation, especially in public networks with rapidly changing demands. This research direction will help bridge the gap between theoretical results and practical performance, demonstrating the feasibility of the proposed approach in real-world scenarios.

\bibliographystyle{IEEEtran} 
\bibliography{main}

\begin{IEEEbiography}[{\includegraphics[width=1in,height=1.25in,clip,keepaspectratio]{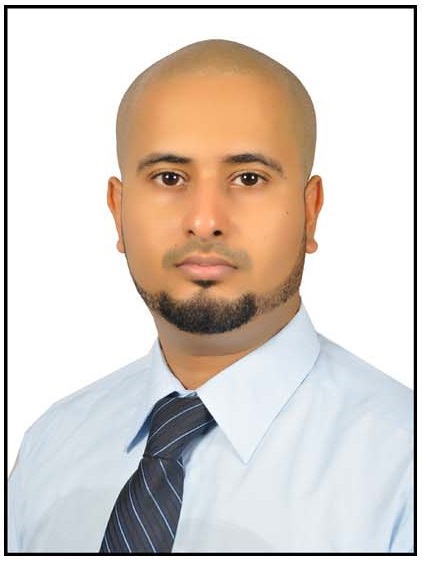}}]{Maad Ebrahim} received his Ph.D. degree in 2024 from the Department of Computer Science and Operations Research (DIRO), University of Montreal, Canada. He received his M.Sc. degree in 2019 from the Computer Science Department, Faculty of Computer and Information Technology, Jordan University of Science and Technology, Jordan. His B.Sc. degree in Computer Science and Engineering was received from the University of Aden, Yemen, in 2013. He is currently working as a machine learning intern at Ericsson GAIA AI Hub, Montreal, Canada. His research experience includes Computer Vision, Artificial Intelligence, Machine learning, Deep Learning, Data Mining, and Data Analysis. His current research interests include Fog and Edge Computing technologies, IoT, Reinforcement Learning, and Blockchains.
\end{IEEEbiography}

\begin{IEEEbiography}[{\includegraphics[width=1in,height=1.25in,clip,keepaspectratio]{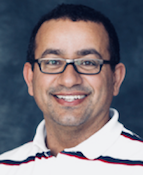}}]{Abdelhakim Senhaji Hafid}
spent several years as the Senior Research Scientist with Bell Communications Research (Bellcore), NJ, USA, working in the context of major research projects on the management of next generation networks. He was also an Assistant Professor with Western University (WU), Canada, the Research Director of Advance Communication Engineering Center (venture established by WU, Bell Canada, and Bay Networks), Canada, a Researcher with CRIM, Canada, the Visiting Scientist with GMD-Fokus, Germany, and a Visiting Professor with the University of Evry, France. He is currently a Full Professor with the University of Montreal. He is also the Founding Director of the Network Research Laboratory and Montreal Blockchain Laboratory. He is a Research Fellow with CIRRELT, Montreal, Canada. He has extensive academic and industrial research experience in the area of the management and design of next generation networks. His current research interests include the IoT, Fog/Edge Computing, blockchain, and intelligent transport systems.
\end{IEEEbiography}

\vfill

\end{document}